\documentclass[10pt,leqno]{article}
\usepackage{lmodern}
\usepackage{authblk}
\usepackage{natbib}
\bibpunct{(}{)}{;}{a}{,}{,}
\usepackage[hidelinks]{hyperref}

\usepackage{lmodern}
\usepackage{geometry}
\usepackage[utf8]{inputenc} 
\usepackage[T1]{fontenc}    
\usepackage{url}            
\usepackage{booktabs}       
\usepackage{amsfonts}       
\usepackage{amsthm}         
\usepackage{nicefrac}       
\usepackage{microtype}      
\usepackage{caption} 
\usepackage{xcolor}         
\usepackage{authblk} 
\usepackage{amsmath}
\usepackage{amssymb}
\usepackage{mathtools}

\usepackage{paralist}

\theoremstyle{plain}
\newtheorem{theorem}{Theorem}[section]
\newtheorem{proposition}[theorem]{Proposition}
\newtheorem{lemma}[theorem]{Lemma}

\theoremstyle{definition}
\newtheorem{definition}[theorem]{Definition}

\theoremstyle{remark}
\newtheorem{remark}[theorem]{Remark}

\newcommand{\R}{\mathbb{R}}

\newcommand{\E}{\mathbb{E}}

\newcommand{\T}{\textsc{T}}
\newcommand{\pinv}{+}

\renewcommand{\vec}[1]{\mathbf{\boldsymbol{#1}}}
\newcommand{\mat}[1]{\mathbf{#1}}

\newcommand{\rank}{\operatorname{rank}}

\newcommand{\diag}{\operatorname{diag}}
\newcommand{\trace}{\operatorname{Tr}}

\newcommand{\metric}[3]{\left\langle #1,#2 \right\rangle_{#3}}

\newcommand{\Xnorm}[2]{\left\| #1 \right\|_{#2}}

\newcommand{\Tnorm}[1]{\Xnorm{#1}{2}}
\newcommand{\XnormS}[2]{\left\| #1 \right\|^2_{#2}}

\newcommand{\TnormS}[1]{\XnormS{#1}{2}}

\newcommand{\at}[2][]{#1\Big|_{#2}}  

\newcommand{\netN}{\mathcal{N}}

\newcommand{\calN}{\mathcal{N}}
\newcommand{\calX}{\mathcal{X}}
\newcommand{\calY}{\mathcal{Y}}

\newcommand{\calP}{\mathcal{P}}

\global\long\def\psd#1{\mathbb{S}_{#1}^{+}}%
\global\long\def\sym#1{\mathbb{S}_{#1}}%

\newcommand{\ve}{\vec{e}}

\newcommand{\vr}{\vec{r}}

\newcommand{\vw}{\vec{w}}
\newcommand{\vx}{\vec{x}}
\newcommand{\vy}{\vec{y}}
\newcommand{\vz}{\vec{z}}

\newcommand{\vtheta}{\vec{\theta}}
\newcommand{\veta}{\vec{\eta}}
\newcommand{\vxi}{\vec{\xi}}

\newcommand{\vrho}{\vec{\rho}}

\newcommand{\matA}{\mat{A}}
\newcommand{\matB}{\mat{B}}
\newcommand{\matC}{\mat{C}}
\newcommand{\matD}{\mat{D}}

\newcommand{\matF}{\mat{F}}

\newcommand{\matH}{\mat{H}}

\newcommand{\matP}{\mat{P}}

\newcommand{\matS}{\mat{S}}

\newcommand{\matI}{\mat{I}}
\newcommand{\matJ}{\mat{J}}

\newcommand{\matV}{\mat{V}}

\newcommand{\matX}{\mat{X}}
\newcommand{\matY}{\mat{Y}}

\newcommand{\matSigma}{\mat{\Sigma}}

\newcommand{\matLambda}{\mat{\Lambda}}
\newcommand{\matTheta}{\mat{\Theta}}

\title{Flatness After All?}

\usepackage{authblk}

\author[ ]{Neta Shoham}
\author[ ]{Liron Mor-Yosef}
\author[ ]{Haim Avron}
\affil[ ]{School of Mathematical Sciences, Tel Aviv University, Tel Aviv, Israel}
\affil[ ]{\texttt{shohamne@mail.tau.ac.il}, \texttt{lm2@mail.tau.ac.il}, \texttt{haimav@tauex.tau.ac.il}}

\begin{document}
\maketitle

\begin{abstract}
Recent literature generalization in deep learning has examined the relationship between the curvature of the loss function at minima and generalization, mainly in the context of overparameterized neural networks. A key observation is that ``flat'' minima tend to generalize better than ``sharp'' minima. While this idea is supported by empirical evidence, it has also been shown that deep networks can generalize even with arbitrary sharpness, as measured by either the trace or the spectral norm of the Hessian. In this paper, we argue that generalization could be assessed by measuring flatness using a soft rank measure of the Hessian. We show that when  
\textcolor{black}{an exponential family neural network} model is exactly calibrated\textcolor{black}{, and its prediction error and its confidence on the prediction are not correlated with the first and the second derivative of the network's output}, our measure accurately captures the asymptotic expected generalization gap. For non-calibrated models, we connect a soft rank based flatness measure to the well-known Takeuchi Information Criterion and show that it still provides reliable estimates of generalization gaps for models that are not overly confident. Experimental results indicate that our approach offers a robust estimate of the generalization gap compared to baselines.
\end{abstract}

\section{Introduction}
\label{sec:introduction}

Recent literature has explored the relationship between the curvature of the loss function at minima and generalization, especially in the context of overparametrized networks. A recurring theme is that ``flat'' minima generalize better than ``sharp'' minima. In the literature, flatness of the minima is typically assessed via a spectral function of the Hessian  \citep{andriushchenko2023modern,ding2024flat,gatmiry2024inductive,liu2023same, petzka2021relative}. Initially, the spectral norm of the Hessian was considered a good gauge for generalization, but \citet{andriushchenko2023modern} have shown experimentally that in some cases the maximum eigenvalue and the generalization gap are negatively correlated. In contemporary literature, the leading metric is the trace of the Hessian at the minima. 

There are several recent works that attempt to justify measuring the generalization gap using the trace of the Hessian.  \citet{ding2024flat, gatmiry2024inductive}  explore the topic in the context of low-rank matrix factorizations. \citet{liu2023same} demonstrated experimentally that there is a strong correlation between the trace of the Hessian of the loss during pre-training and the gap of the model on downstream tasks. They further prove a theoretical result for a special case of synthetic upstream and downstream tasks.

The line of research connecting the generalization gap to flatness, along with the empirical success in deep learning of using flatness measures to predict generalization, is somewhat surprising if we consider that it has been shown that deep networks can generalize well with arbitrary flatness (i.e., also with sharp minima), when flatness is measured by the trace of the Hessian or the spectral norm of the Hessian \citep{dinh2017sharp}. Furthermore, pursuing models with small Hessian trace stands in contrast to the well-known A-optimal design criterion, which seeks designs that minimize the trace of the inverse of the Hessian (i.e., sharp minima). There is also a vast literature that promotes whitening of representations in deep learning \citep{cogswell2015reducing, hua2021feature, lee2023importance}, which aligns more with sharpening than flattening.
Contrastive learning aims to produce orthogonal representations for samples that are dissimilar, and so is another form of whitening, and thus sharpening.

How can we reconcile the discrepancy between results that advocate the use of flatness as a measure of generalization and results that cast doubt on this notion? Should we abandon the use of flatness as an explanatory factor for the generalization gap? The main premise of this paper is that {\em no, we should not}, although we do need to adjust how flatness is measured.  We argue that to assess generalization, we should measure flatness via the following {\em soft rank} of the Hessian $\matH(\vtheta^\star)$ at a local minimum\footnote{This quantity is also known in the literature as {\em statistical dimension} and {\em effective number of dimensions}. We chose to use the term {\em soft rank} since the connection of this quantity to matrix rank is more central to our arguments.}:
\begin{equation*}
	{\rank}_\lambda(\matH(\vtheta^\star)) \coloneqq \trace(\matH(\vtheta^\star)(\matH(\vtheta^\star)+\lambda \matI)^{-1})
\end{equation*}
where $\lambda$ is the weight decay coefficient. ${\rank}_\lambda(\cdot)$ is a monotone increasing function of positive semidefinite matrices and, as such, deserves to be viewed as a flatness measure. 
 
For an exponential family neural network model (see Section~\ref{subsec:exp-family} for a definition) with prediction error and confidence that are not correlated with the network output's first and second derivative, we show that if there exists a weight decay coefficient $\lambda$ in which at a local minima $\vtheta^\star$ the model is {\em calibrated}\footnote{Informally, a model is calibrated if, given its predicted label distribution, the true labels follow the same distribution as the prediction.
}, then at that $\vtheta^\star$ we have that $\rank_\lambda(\matH(\vtheta^\star))$ on the theoretical population distribution is {\em exactly} the  asymptotic expected generalization gap.  One might argue that even if such a $\lambda$ exists, finding it is not practical. Further, even if we do find it, why would we believe that this specific calibrated model is the best model? Even though calibration is a desired property, much like unbiasedness it is not necessarily a trait of the optimal model. Obviously, these are valid concerns. The point is, {\em if} such a $\lambda$ exists, the soft rank is the correct measure of flatness, and other quantities are {\em incorrect}. Thus, it makes sense to analyze the use of soft rank as a flatness measure also for other cases. 

Indeed, in real-world training of deep networks, especially with limited data, one often finds minima in a region where the model is not only misspecified, but also uncalibrated. In such cases, curvature is not the only factor that determines the generalization gap.  Several recent works \citep{thomas2020interplay, naganuma2022takeuchis, jang2022reparametrization}  leveraged the classical \emph{Takeuchi Information Criterion} (TIC) \citep{takeuchi1976distribution}, which shows that asymptotically, the generalization gap equals $n^{-1}\cdot\trace(\matC(\vtheta^\star)\matH(\vtheta^\star)^{-1})$, where $n$ is the number of training samples and $\matC(\vtheta^\star)$ is the gradient covariance at $\vtheta^\star$. \citet{thomas2020interplay} demonstrated empirically that in practice $\matC(\vtheta^\star) \approx \alpha \matH(\vtheta^\star)$ for some $\alpha>0$, and thus claimed that the generalization gap can be estimated as $n^{-1}\cdot \nicefrac{\trace(\matC(\vtheta^\star))}{\trace(\matH(\vtheta^\star))}$. However, two crucial points were missed by previous literature, which we remedy:
\begin{itemize}
    \item TIC only works with the negative log-likelihood loss. More general versions of TIC allow general loss functions  \citep{Shibata1989, murata1994network, liu1995unbiased}, but with these the situation becomes more complex. Based on \citet{thomas2020interplay}'s proportionality observation, we argue that the generalization gap is now the product of two components: $\nicefrac{\trace(\matC(\vtheta^\star))}{\trace(\matH(\vtheta^\star))}$ {\em and} the soft rank of the Hessian at $\vtheta^\star$, and not only the trace ratio as advocated by \citet{thomas2020interplay}.
	
    \item Estimating $\nicefrac{\trace(\matC(\vtheta^\star))}{\trace(\matH(\vtheta^\star))}$ using training data alone is essentially impossible. The reason is that this quantity amounts to the ratio  between accuracy and average uncertainty.  As such, under overfitting (which is expected in overparameterized deep networks), the ratio tends to $0$ on the training set, while it tends to $\infty$ on the population if mistakes are made with high confidence. The logical conclusion is that we cannot estimate the generalization gap under overfitting, which is in line with the observations made by \citet{gastpar2024fantastic}. In contrast, the soft rank can be estimated quite robustly from training data alone.  Thus, for cases where $\nicefrac{\trace(\matC(\vtheta^\star))}{\trace(\matH(\vtheta^\star))}\approx 1$, we can estimate the generalization gap quite robustly, and in any case, we can reasonably estimate the asymptotic behavior of the generalization gap as long as the ratio between accuracy and average uncertainty remains bounded.
\end{itemize}

Seemingly, our results contradict the findings of \citet{dinh2017sharp}, which state that good minima of feed-forward networks can be arbitrarily sharp. However, there is no contradiction. The result in \citep{dinh2017sharp} is based on a sophisticated layer rescaling argument. However, this argument does not hold when regularization is involved. The reason is that even if the log-likelihood is invariant to such rescaling, the regularization term is not. Even if we take the regularization parameter to $0$, the soft rank converges to the rank of the Hessian, which is again invariant to rescaling. The bottom line is that when learning with weight decay, good minima {\em cannot} be arbitrarily sharp as long as the model is not arbitrarily confident on mistakes.

\subsection{Additional Related Work}
\citet{petzka2021relative} proposed the use of 
$$\|\vw^\star\|^2 \trace\left(\nabla^2_{\vw^\star} \frac{1}{n}\sum_{i=1}^n l(f(\phi(\vx_i)^T \vw^\star), \vy_i) + C\right),$$
(where $\phi(\vx)$ is the representation learned by a deep network) as a scale invariant measure of relative flatness, and connected it to generalization. However, while their measure is invariant to scaling of the parameters $\vw^\star$, it is not invariant to rescaling of the feature embedding $\phi(\vx)$.

In the Bayesian approach, instead of the expected population log-likelihood, one measures the marginal likelihood. A tractable approximation of the marginal likelihood for deep networks is possible via the Laplace's approximation.  This results in an expression similar to the expected log-likelihood, but with a regularized log determinant of the Hessian instead of its soft rank \citep{mackay1992bayesian, mackay1992thesis}. Both the soft rank $\rank_\lambda(\cdot)$ and the log determinant are increasing concave functionals on the semidefinite cone and thus they deliver comparable performance as regularizers and admit similarly efficient computational approximations. \citet{immer2021improving, immer2021scalable, immer2022invariance, immer2023stochastic} recently advocated the use of such functionals for model and hyperparameter selection.

It is customary to measure the effective dimension of learned deep representations via the soft rank of the representation covariance matrix. This measure is closely connected to the soft rank of the Fisher Information Matrix and Hessian. It has been shown that using SGD tends to find  representations with small effective dimension \citep{arora2019implicit, razin2020implicit, baratin2021implicit, huh2023the}. However, these works simply assume a connection between low effective dimension and generalization, without proving this. Furthermore,  they measure the effective dimension in a different, but subtle, way: as a measure of the entropy of the eigenvalues (or the normalized eigenvalues) of the representations covariance.

\citet{moody1991effective} referred to the soft rank of the empirical covariance of the gradients of the regularized square loss of the learned model as the effective number of parameters, and suggested it as a model selection criterion for non-linear regression.

\section{Preliminaries}

\subsection{Regularized Maximum Likelihood Estimation}
\label{subsec:regularized-mle}
Our analysis is based on viewing the learned parameters $\vtheta$ of a deep network as a regularized Maximum Likelihood Estimator (MLE). 

In regularized MLE, we assume data $\vz_1,\dots,\vz_n$ is sampled i.i.d from an unknown {\bf data distribution} $p_{\vz}(\cdot)$. We then set up a {\bf model distribution} $q(\cdot|\vtheta)$ parameterized by $\vtheta\in\matTheta\subseteq \R^d$. We learn the model by computing a minimizer $\hat{\vtheta}$ of 
\[
\smash{L_{\kappa}(\vtheta)\coloneqq \frac{1}{n}\sum^n_{i=1}-\log q(\vz_i|\vtheta)+\kappa(\vtheta) \coloneqq \frac{1}{n}\sum^n_{i=1} l_\kappa(\vtheta;\vz_i)}
\]
where $\kappa(\cdot)$ is a regularization function.  

Under certain regularity conditions, $\hat{\vtheta}$ is a consistent estimator of $\vtheta^\star\coloneqq \arg\min_{\vtheta\in\matTheta} \E_\vz\left[l_\kappa(\vtheta;\vz) \right]$, i.e., $\hat{\vtheta}\to\vtheta^\star$ in probability as $n\to \infty$. In the previous line, we made the simplifying assumption that $\matTheta$ is small enough so that $\E_\vz\left[l_\kappa(\cdot;\vz) \right]$ has a unique minimum. However, in terms of the likelihood (which in deep learning corresponds to the loss), the optimized loss is too optimistic with respect to population loss. Formally, define
$$
\text{Gap} \coloneqq \smash{\E_\vz\left[l_\kappa(\hat{\vtheta};\vz) \right]} - L_{\kappa}(\hat{\vtheta})
$$
Note that $\text{Gap}$ depends on the $\hat{\vtheta}$, which, in turn, depends on the training set $\vz_1,\dots,\vz_n$, and so is a random variable.\footnote{To avoid clutter, we chose not to make this dependency explicit.}
Then, it is straightforward to show that 
$$
\smash{\E_{\vz_1,\dots,\vz_n}}\left[\text{Gap}\right] \geq 0
$$

The regularized MLE framework is typically applied to deep learning in the following way. First, we assume we can write $\vz=(\vx,\vy)\in\calX \times \calY$, and the data distribution is factorized $p_\vz(\vz)=p_{\vy|\vx}(\vy)p_\vx(\vx)$. We factorize the model distribution as well: $q(\vz|\vtheta)=q(\vy|\vx,\vtheta)p_\vx(\vx)$. The distribution  $q(\cdot|\vx,\vtheta)$ is set to be
\[
q(\vy|\vx,\vtheta) = q_{\vy\mid\netN}(\vy|\netN_\vtheta(\vx)),
\]
where $q_{\vy\mid\netN}(\cdot|\netN_\vtheta(\vx))$ a member of some parametric family of distributions, with the parameter set to be the output $\netN_\vtheta (\vx)$ of the neural network. Different parametric families lead to different losses $L_{\kappa}(\vtheta)$ for learning the neural network.

\subsection{Exponential family neural network model}\label{subsec:exp-family}
We further assume that this parametric family is an exponential family, i.e.,
\[
q_{\vy\mid\calN}(\vy\mid\calN) = h(\vy)e^{\left(\calN^\T\vy-A(\calN)\right)}.
\] 
We call $q(\cdot|\vtheta)$ of the form $q(\vz|\vtheta)=q_{\vy\mid\calN}(\vy\mid\calN_\theta(x))p_\vx(\vx)$ an \emph{exponential family neural network model}. For example, setting $q_{\vy\mid\calN}(\vy \mid \calN_\theta(x))$ to a normal distribution with mean $\netN_\vtheta(\vx)$ results in the squared loss, while setting it to a binomial distribution with probability $\netN_\vtheta(\vx)$ results in the categorical cross-entropy loss for classification.

In our analysis, we use the following function, which we call ``cost function'':
\[
c(\vy,\calN) \coloneqq -\log q_{\vy\mid\calN}(\vy\mid\calN)=-\calN^\T\vy+A(\calN)-\log h(\vy).
\]
For example, in regression with known noise variance $\matI$, we have 
\[
c(\vy,\calN)=\frac{1}{2}\TnormS{\vy-\calN}+\text{const} = -\calN^\T\vy+\frac{1}{2}\TnormS{\calN}+\text{const},
\]
that is $A(\calN)=\frac{1}{2}\TnormS{\calN}$. In classification with softmax cross-entropy cost we have,
\[
c(\vy,\calN)=-\calN^\T\vy+\sum_i e^{\calN_i},
\]
and $A(\calN)=\sum_i e^{\calN_i}$.

\subsection{Information Matrices}
We recall three (regularized) information matrices, whose exact relations have been the source of confusion among researchers \citep{thomas2020interplay,kunstner2019limitations}:
\begin{align*}
\text{Covariance:} &\quad \matC_\kappa(\vtheta)\coloneqq \E_{\vz\sim p_z} \left[\nabla_\vtheta l_\kappa(\vtheta;\vz)\nabla_\vtheta l_\kappa(\vtheta;\vz)^\T\right] \\
\text{Fisher Information:} &\quad \matF(\vtheta)\coloneqq \E_{\vz\sim q_\vz(\cdot|\vtheta)} \left[\nabla_\vtheta l_0(\vtheta;\vz)\nabla_\vtheta l_0(\vtheta;\vz)^\T\right] \\
\text{Hessian:} &\quad \matH_\kappa(\vtheta)\coloneqq \E_{\vz\sim p_z} \left[\nabla^2_\vtheta l_\kappa(\vtheta;\vz)\right]
\end{align*}

Information matrices play a key role in analyzing the generalization of MLEs. \citet{Shibata1989} analyzed the asymptotic behavior of the expected generalization gap (additional references are \citep{murata1994network, liu1995unbiased}):
\begin{equation}
    \label{eq:shibata}
    \lim_{n\to\infty} n\cdot\E\left[\text{Gap}\right]=\trace(\matC_\kappa(\vtheta^\star)\matH_\kappa(\vtheta^\star)^{-1})
\end{equation}
Eq.~\eqref{eq:shibata} is a generalization of TIC (which applies only to unregularlized MLE), which is, in turn, a generaliation of {\em Akaike Information Criterion} (AIC), which applies only to correctly-specified models~\citep{akaike1974new}.

Let us denote by $\matC(\vtheta), \matF(\vtheta)$ and $\matH(\vtheta)$ the corresponding information matrices when no regularization is involved ($\kappa(\vtheta)=0$). When using Tikhonov (aka ridge) regularization, i.e., $\kappa(\vtheta)=\frac{1}{2}\TnormS{\matLambda^{\nicefrac{1}{2}} \vtheta}$ for some symmetric positive definite $\matLambda$, we denote the information matrices by $\matC_{\matLambda}(\vtheta)$ and $\matH_{\matLambda}(\vtheta)$. We have the following:

\begin{proposition}
\label{prop:C_Lambda}
If $\vtheta^\star$ is a local minimizer,
\begin{align}
\textstyle
    \matH_\matLambda(\vtheta) &= \matH(\vtheta) + \matLambda \nonumber \\
    \matC_{\matLambda}(\vtheta^\star) & = \matC(\vtheta^\star) -\matLambda \vtheta^\star {\vtheta^\star}^\T \matLambda   \label{eq:C_Lambda}
\end{align}
\end{proposition}
\begin{proof}
\label{app:proof_C_Lambda}

The first equality follows from straightforward matrix calculus.

Since $\vtheta^\star$ is a local minimizer, $$\nabla_\vtheta \E_{\vz\sim p_\vz}\left[l_0(\vtheta;\vz) + \frac{1}{2}\TnormS{\matLambda^{\nicefrac{1}{2}} \vtheta}\right]\at{\vtheta=\vtheta^\star}=0$$ we have  $$\nabla_\vtheta \E_{\vz\sim p_\vz}\left[l_0(\vtheta;\vz)\right]\at{\vtheta=\vtheta^\star} = -\matLambda \vtheta^\star.$$ Denote $g(\vz) \coloneqq \nabla_\vtheta l_0(\vtheta;\vz)\at{\vtheta=\vtheta^\star}$, so we have $$\matC(\vtheta^\star) = \E_\vz [g(\vz)g(\vz)^\T]$$ and $$\E_\vz[g(\vz)] = \nabla_\vtheta \E_{\vz\sim p_\vz}\left[l_0(\vtheta;\vz)\right]\at{\vtheta=\vtheta^\star} = -\matLambda \vtheta^\star$$ Thus:
\begin{equation*}
    \matC_\matLambda(\vtheta^\star) = \E_{\vz\sim p_z} \left[\left(g(\vz) + \matLambda \vtheta^\star\right)\left( g(\vz) + \matLambda \vtheta^\star \right)^\T\right] = \matC(\vtheta^\star) -\matLambda \vtheta^\star {\vtheta^\star}^\T \matLambda
\end{equation*}
\end{proof}

\subsection{Calibrated Neural Networks}

A model is \emph{calibrated} if predicted probabilities match actual probabilities conditioned on predictions. Formally:

\begin{definition}[Calibrated Neural Network Model]
\label{definition:calibrate}
Given $\vtheta$, let $\vr_\vtheta \coloneqq \calN_\vtheta(\vx)$. The neural network model is \emph{calibrated} at $\vtheta$ if
\[
p_{\vy\mid\vr_\vtheta}(\cdot\mid\vr_\vtheta) = q_{\vy\mid\calN}(\cdot\mid\vr_\vtheta)
\]
almost surely (a.s. in $\vx$).
\end{definition}

For correctly specified or calibrated models, we show that the covariance of gradients, Fisher Information, and expected Hessian coincide if prediction errors and uncertainties are uncorrelated with derivatives of network outputs:
\begin{equation}
\label{eq:uncor_err}
\begin{gathered}
\E[\matJ_{\vtheta}^\T \ve_\vtheta\ve_\vtheta^\T \matJ_{\vtheta}]
= \E[\matJ_{\vtheta}^\T \E[\ve_\vtheta\ve_\vtheta^\T] \matJ_{\vtheta}], \quad
\E[\matJ_{\vtheta}^\T \matSigma_\vtheta \matJ_{\vtheta}]
= \E[\matJ_{\vtheta}^\T \E[\matSigma_\vtheta] \matJ_{\vtheta}], \\
\E[\ve_{\vtheta,i}\nabla^2_\vtheta \calN_{\vtheta,i}(\vx)]
= \E[\ve_{\vtheta,i}] \E[\nabla^2_\vtheta \calN_{\vtheta,i}(\vx)],
\end{gathered}
\end{equation}
where $\matJ_{\vtheta}$ is the Jacobian of $\calN_\vtheta(\vx)$,
\begin{equation}   \label{eq:e_theta}
\ve_\vtheta \coloneqq \hat{\vy}_\vtheta-\vy,
\quad
\hat{\vy}_\vtheta\coloneqq \E_{\vy\sim q_{\vy\mid\netN}(\cdot|\netN_\vtheta(\vx))}[\vy]
\quad
 \text{and} 
\quad
\matSigma_\vtheta \coloneqq \mathrm{Var}_{\vy\sim q_{\vy\mid\netN}(\cdot|\netN_\vtheta(\vx))}[\vy].
\end{equation}

(An empirical evidence for these noncorrelations is provided in Section~\ref{app:experimental_results_CF}.)

\begin{proposition}
\label{prop:calibrated_model}
If $q_{\vy\mid\calN}$ is an exponential family, Eqs. \eqref{eq:uncor_err} hold, and the model is calibrated at $\vtheta$, then
\[
\matC(\vtheta) = \matF(\vtheta) = \matH(\vtheta).
\]
\end{proposition}
\begin{proof}
By the chain rule, and  our assumptions, 
\begin{align*}
\matC(\vtheta) 
&= \E\left[\matJ_{\vtheta}(\vx)^\T \ve_\vtheta(\vx,\vy)\ve_\vtheta(\vx,\vy)^\T \matJ_{\vtheta}(\vx)\right]\\ 
&= \E\left[\matJ_{\vtheta}(\vx)^\T \E\left[\ve_\vtheta(\vx,\vy)\ve_\vtheta(\vx,\vy)^\T\right] \matJ_{\vtheta}(\vx)\right].
\end{align*}
By the chain rule, law of iterated expectation, and our assumptions,
\begin{align*}
\matF(\vtheta) 
&= \E_{\vz\sim q(\cdot|\vtheta)}\left[\matJ_\vtheta(\vx)^\T \ve_\vtheta(\vx,\vy)\ve_\vtheta(\vx,\vy)^\T \matJ_\vtheta(\vx)\right]\\ 
&= \E_{\vx\sim p_\vx}\left[\matJ_\vtheta(\vx)^\T \E_{\vy\sim{q_{\vy\mid\calN}}(\cdot \mid \calN)}\left[\ve_\vtheta(\vx,\vy)\ve_\vtheta(\vx,\vy)^\T\right] \matJ_\vtheta(\vx)\right]\\
&= \E\left[\matJ_\vtheta(\vx)^\T \matSigma_\vtheta(\vx) \matJ_\vtheta(\vx)\right]\\
&=\E\left[\matJ_\vtheta(\vx)^\T \E\left[\matSigma_\vtheta(\vx)\right] \matJ_\vtheta(\vx)\right].
\end{align*}
Define,
\[
\matB(\vy,\calN) \coloneq \nabla_\calN c(\vy,\calN)\nabla_\calN c(\vy,\calN)^\T\quad\text{and}\quad\bar{\matSigma}(\calN)\coloneqq \nabla^2_\calN c(\vy,\calN).
\]
Let
\(
\vr_\vtheta\coloneq\calN_\vtheta(\vx).
\) 
Note that 
\[
\quad\matB(\vy,\vr_\vtheta)=\matB(\vy,\calN_\vtheta(\vx))=\ve(\vx,\vy)\ve(\vx,\vy)^T\quad\text{and}\quad\bar{\matSigma}(\vr_\vtheta)=\bar{\matSigma}(\calN_\vtheta(\vx))=\matSigma_\vtheta(\vx)
\]
By the calibration assumption,
\[
p_{\vy\mid\vr_\vtheta}(\cdot\mid\vr_\vtheta)=q_{\vy\mid\calN}(\cdot\mid\vr_\vtheta),
\]
where $p_{\vy\mid\vr_\vtheta}(\cdot\mid\vr_\vtheta)$ is the real distribution of $\vy$ given $\vr_\vtheta$. Since $q_{\vy\mid\calN}$ is of exponential family we also have
\[
\E_{\vy\sim{q_{\vy\mid\calN}(\vy\mid\vr_\vtheta)}}\matB(\vy,\vr_\vtheta)=\bar{\matSigma}(\vr_\vtheta).
\]
Now, by total expectation,
\begin{align*}
\E\left[\ve_\vtheta(\vx,\vy)\ve_\vtheta(\vx,\vy)^\T\right]
&= \E\left[\matB(\vy,\vr_\vtheta)\right] \\
&= \int\int \matB(\vy,\vr_\vtheta)\, p_{\vy|\vr_\vtheta}(\vy\mid\vr_\vtheta)p_{\vr_\vtheta}(\vr_\vtheta)\, d\vy\, d\vr_\vtheta\\
&= \int\int \matB(\vy,\vr_\vtheta)\, q_{\vy\mid\calN}(\vy\mid\vr_\vtheta)p_{\vr_\vtheta}(\vr_\vtheta)\, d\vy\, d\vr_\vtheta\\
&= \int p_{\vr_\vtheta}(\vr_\vtheta) \left(\int \matB(\vy,\vr_\vtheta)\, q_{\vy\mid\calN}(\vy\mid\vr_\vtheta)\, d\vy\right) d\vr_\vtheta\\
&= \E\left[\bar{\matSigma}_\vtheta(\vr_\vtheta)\right] \\
&= \E\left[\matSigma_\vtheta(\vx)\right], \\
\end{align*}
and thus, \(\matC(\vtheta)=\matF(\vtheta)\).

For $\matH(\vtheta)=\matF(\vtheta)$, recall that for an exponential family,
\[
\matSigma_\vtheta(\vx) \coloneq \nabla^2_\calN c(\vy,\calN_\vtheta(\vx))=  \E_{\vy\sim{q_{\vy\mid\calN}}(\cdot\mid\calN_\vtheta(\vx))}\left[ \ve_\vtheta(\vx,\vy)\ve_\vtheta(\vx,\vy)^\T \right].
\]
Thus, as $\nabla l_\vtheta(\vtheta;\vx,\vy)=\matJ_{\vtheta}(\vx)^\T\ve(\vx,\vy)$, by the chain rule
\[
\matH(\vtheta)=\E\left[\matJ_\vtheta(\vx)^\T\matSigma_\vtheta(x)\matJ_\vtheta(\vx)\right] +\sum_i \E\left[\ve_{\vtheta,i}(\vx,\vy) \nabla^2_\theta\calN_{\vtheta,i}(\vx)\right]
=\matF(\vtheta)+\sum_i \E\left[\ve_{\vtheta,i}(\vx,\vy) \nabla^2_\theta\calN_{\vtheta,i}(\vx)\right]. 
\]
Now, assume that for some $\alpha\in\R$,
\[
\E\left[\ve_{\vtheta,i}(\vx,\vy) \nabla^2_\theta\calN_{\vtheta,i}(\vx)\right]=\alpha\E\left[\ve_{\vtheta,i}(\vx,\vy)\right] \E\left[\nabla^2_\theta\calN_{\vtheta,i}(\vx)\right].
\]
(weaker then the assumption in the Eq.~\eqref{eq:uncor_err}). Note that
\[
\ve_\vtheta(\vx,\vy) = -\nabla_\calN \log q_{\vy\mid\calN}(\vy, \calN_\vtheta(\vx)) = -\nabla_\calN \log q_{\vy\mid\calN}(\vy, \vr_\vtheta). 
\]
Now, since $\vr_\vtheta$ is a true parameter of $q_{\vy\mid\calN}$, we have that
\[
\E_{\vy\sim{q_{\vy\mid\calN}}(\cdot \mid \vr_\vtheta)}\left[-\nabla_\calN \log q_{\vy\mid\calN}(\vy\mid \vr_\vtheta)  \right]=0,
\]
and thus, by total expectation under the calibration assumption we have,
\begin{align*}
\E\left[\ve_\vtheta(\vx,\vy)\right]
&=\E\left[-\nabla_\calN \log q_{\vy\mid\calN}(\vy\mid \vr_\vtheta) \right]\\
&=\E_{\vr_\vtheta}\left[\E_{\vy\sim{p_{\vy\mid\vr_\vtheta}}(\vy\mid \vr_\vtheta)}\left[-\nabla_\calN \log q_{\textbf{}}(\vy\mid \vr_\vtheta) \right]\right]\\
&=\E_{\vr_\vtheta}\left[\E_{\vy\sim{q_{\vy\mid\calN}}(\vy\mid \vr_\vtheta)}\left[-\nabla_\calN \log q_{\vy\mid\calN}(\vy\mid \vr_\vtheta) \right]\right]\\
&=0.
\end{align*}
\end{proof}
\subsection{Soft Ranks and Soft Projections}

For a positive semidefinite matrix $\matS$ and a positive definite matrix $\matLambda$, we use the notations $\matP_\matLambda(\matS) \coloneqq \matS(\matS+\matLambda)^{-1}$ and $\matP^{\perp}_\matLambda(\matS) \coloneqq \matLambda(\matS+\matLambda)^{-1}$. We have:
\begin{equation}
\label{eq:P_Lambda}
\mat0\preceq\matP_\matLambda(\matS)^\T\matP_\matLambda(\matS) = (\matS+\matLambda)^{-1}\matS^2(\matS+\matLambda)^{-1} \prec \matI \quad \text{and} \quad \matI - \matP_\matLambda(\matS)=\matP^{\perp}_\matLambda(\matS).
\end{equation}
Thus:
\begin{equation}
    \label{eq:P_bound}
    \Tnorm{\matP_\matLambda(\matS)}< 1 \qquad \text{and} \qquad \Tnorm{\matP^{\perp}_\matLambda(\matS)} \le 1
\end{equation}
In the limit of $\matLambda\to \mat0$, $\matP_\matLambda$ is the projection onto the eigenspace of $\matS$ and thus can be seen as a form of ``soft projection''. Similarly, $\rank_\matLambda(\matS) \coloneqq \trace (\matP_\matLambda(\matS))$ can be thought of as a soft rank of $\matS$. We also use $\rank^{(2)}_\matLambda(\matS) \coloneqq \trace (\matP_\matLambda(\matS)^2)$. Note that:
\begin{equation}
\lim_{\matLambda \to 0}\rank^{(2)}_\matLambda(\matS) =  \lim_{\matLambda \to 0}\rank_\matLambda(\matS) = \rank(\matS).
\end{equation}
For $\lambda>0$, we define $\rank_\lambda\coloneq\rank_{\lambda\matI}$ and $\rank^{(2)}_\lambda\coloneq\rank^{(2)}_{\lambda\matI}$. 

\section{Sharp or Flat? On the Importance of Selecting an Appropriate Inner Product}

``Sharpness'' and ``flatness'', and more generally curvature, are geometric notions. Thus, when analyzing the curvature of the loss function it is not enough to specify the parameter domain (in our case: $\matTheta$). We need to also specify how angles and distances are measured, by imposing an inner product, or more generally, if $\matTheta$ is a manifold, a Riemannian metric. Much of the literature on the relation between flatness and generalization loss overlooks this fact, and simply assumes the use of the standard dot product as the inner product. This renders the measures sensitive to reparametrization~\citep{jang2022reparametrization,kristiadi2024geometry}, and allows one to show that arbitrarily ``sharp'' minima can generalize well~\citep{dinh2017sharp}.

Any metric will induce a reparametrization invariant (Riemannian) Hessian of the loss at the minima~\citep{kristiadi2024geometry}. However, when choosing a metric, it is crucial to adjust it to the downstream task, and this can have a drastic effect on whether we want a sharp or flat minima. We illustrate this point with several examples.

\subsection{Parameter Estimation}
Suppose we are interested in estimating the optimal (in the sense of minimal loss on the population) parameters $\vtheta^\star$ (this is {\em not} the situation in deep learning. We only discuss this scenario for illustration purposes). On correctly-specified models, and without regularization, this corresponds to the true parameters, though we do not restrict ourselves to this scenario. A natural metric for the error is then $\|\hat{\vtheta} - \vtheta^\star\|_2$, and this corresponds to using the standard Euclidean inner product. It is well-known that the asymptotic distribution of $\sqrt{n}\cdot(\hat{\vtheta} - \vtheta^\star)$ is normal with zero mean and variance $\matH_\kappa(\vtheta^\star)^{-1}\matC_\kappa(\vtheta^\star)\matH_\kappa(\vtheta^\star)^{-1}$ \citep{newey1994large}. Thus, we have 
\begin{equation}
\label{eq:theta-error}
\lim_{n\to\infty} n\cdot\E\left[\TnormS{\hat{\vtheta} - \vtheta^\star} \right] = \trace(\matC_\kappa(\vtheta^\star)\matH_\kappa(\vtheta^\star)^{-2})
\end{equation}
(this classical analysis leads to the notion of A-optimality). Now, if an exponential family neural network  model is used, with noncorrelations specified in Eqs.~\eqref{eq:uncor_err} is calibrated at $\vtheta^\star$, and there is no regularization, the quantity on the right is equal to $\trace(\matH(\vtheta^\star)^{-1})$. Since the Hessian is inverted here, {\em we see that when estimating the parameters using such a calibrated model with no regularization, we want sharp, not flat, minima!} 

\subsection{Geometry for Generalization}
Assume that for every $\vtheta\in\matTheta$, the function $l_\kappa(\vtheta;\cdot)$ is square integrable with respect to $p_\vz$. We identify each parameter vector $\vtheta$ with the loss function it defines: $l_\kappa(\vtheta;\cdot)$. In this view, the MLE $l_\kappa(\hat{\vtheta};\cdot)$ is an estimator of $l_\kappa(\vtheta^\star;\cdot)$.  Distances between the two can then be measured using the $L_2$ distance with respect to $p_\vz$: 
$$\XnormS{l_\kappa(\hat{\vtheta};\cdot)-l_\kappa(\vtheta^\star;\cdot)}{L_2(p_\vz)} = \E_{\vz\sim p_\vz}\left[\left(\hat{l}(\hat{\vtheta};\vz) - \hat{l}(\vtheta^\star;\vz)\right)^2\right].$$
Clearly, this is a reasonable measure for generalization. 

If distances are measured using the $L_2$ norm, inner products should be defined via the $L_2$ inner product. 
To do so, first assume that $\matTheta$ is an open set, and view $\matTheta$ as a submanifold of $\R^d$. We impose a Riemannian metric on $\matTheta$ in the following way. At a point $\vtheta\in\matTheta$, the tangent space is isomorphic to $\R^d$. For two tangent directions $\veta,\vxi\in\R^d$, the inner product is defined as the $L_2(p_\vz)$ inner product between the infinitesimal perturbation of $l_\kappa(\vtheta;\cdot)$ in the direction (in parameter space) of $\veta$ and the infinitesimal perturbation of $l_\kappa(\vtheta;\cdot)$ in the direction of $\vxi$. Formally, for $\vrho\in\R^d$ let
$$
\delta_\vrho l_\kappa(\vtheta; \vz) \coloneqq l_\kappa(\vtheta+\vrho; \vz) - l_\kappa(\vtheta; \vz)
$$
and define
$$
\metric{\veta}{\vxi}{\vtheta} \coloneqq \lim_{\varepsilon\to 0} \varepsilon^{-2}\cdot\metric{\delta_{\varepsilon\cdot\veta} l_\kappa(\vtheta;\cdot)}{\delta_{\varepsilon \cdot \vxi} l_\kappa(\vtheta;\cdot)}{L_2(p_{\vz})}
$$
By Taylor expansion around $\vtheta$ we have $\delta_\vrho l_\kappa(\vtheta; \vz) = \nabla_\vtheta l_\kappa(\vtheta; \vz)^\T \vrho + o(\Tnorm{\vrho})$.
Now, 
\begin{align*}
\metric{\delta_{\varepsilon\cdot\veta} l_\kappa(\vtheta;\cdot)}{\delta_{\varepsilon \cdot \vxi} l_\kappa(\vtheta;\cdot)}{L_2(p_{\vz})}
&=  \varepsilon^2 \cdot \int   \nabla_\vtheta l_\kappa(\vtheta; \vz)^\T \veta \cdot  \nabla_\vtheta l_\kappa(\vtheta; \vz)^\T \vxi p_{\vz}d\vz + o(\varepsilon^2)\\
&= \varepsilon^2 \cdot \veta^\T \left(\int   \nabla_\vtheta l_\kappa(\vtheta; \vz)\nabla_\vtheta l_\kappa(\vtheta; \vz)^\T p_{\vz}d\vz\right) \vxi + o(\varepsilon^2) \\
&= \varepsilon^2 \cdot \veta^\T \matC_\kappa(\vtheta) \vxi + o(\varepsilon^2)
\end{align*}
Thus, $\metric{\veta}{\vxi}{\vtheta} = \veta^\T \matC_\kappa(\vtheta) \vxi$, and the Riemannian Hessian at $\vtheta^\star$ is $\bar{\matH}_\kappa(\vtheta^\star) \coloneqq \matC_\kappa(\vtheta^\star)^{-1}\matH_\kappa(\vtheta^\star)$ \citep{kristiadi2024geometry}.

Using similar arguments, we have
\begin{align*}
\XnormS{l_\kappa(\hat{\vtheta};\cdot)-l_\kappa(\vtheta^\star;\cdot)}{L_2(p_\vz)}
= (\hat{\vtheta} - \vtheta^\star)^\T \matC_\kappa (\vtheta^\star) (\hat{\vtheta} - \vtheta^\star) + o\left(\TnormS{\hat{\vtheta} - \vtheta^\star}\right)
\end{align*}
Alluding again to the asymptotic normality of $\sqrt{n}\cdot(\hat{\vtheta} - \vtheta^\star)$, we have
\begin{align}
\nonumber
\lim_{n\to\infty}n\cdot\E\left[\XnormS{l_\kappa(\hat{\vtheta};\cdot)-l_\kappa(\vtheta^\star;\cdot)}{L_2(p_\vz)} \right]
\nonumber
&=\lim_{n\to\infty} \E\left[\sqrt{n}(\hat{\vtheta} - \vtheta^\star)^\T \matC_\kappa (\vtheta^\star) \sqrt{n}(\hat{\vtheta} - \vtheta^\star)\right] \\
\nonumber
&\qquad + \E\left[o(\sqrt{n}(\hat{\vtheta} - \vtheta^\star)^\T\sqrt{n}(\hat{\vtheta} - \vtheta^\star))\right]\\ 
\label{eq:l2_limit}
\nonumber
&=\lim_{n\to\infty} \trace(\matC_\kappa(\vtheta^\star)\matH_\kappa(\vtheta^\star)^{-1}\matC_\kappa(\vtheta^\star)\matH_\kappa(\vtheta^\star)^{-1})+o(1)\\
&=\trace(\matC_\kappa(\vtheta^\star)\matH_\kappa(\vtheta^\star)^{-1}\matC_\kappa(\vtheta^\star)\matH_\kappa(\vtheta^\star)^{-1})\\
\nonumber
&= \trace(\bar{\matH}_\kappa(\vtheta^\star)^{-2})
\end{align}
Again, we see that a sharp minima is preferable, if sharpness is measured using the appropriate geometry.

\subsection{Information Geometry View}
We can also take an information geometry viewpoint, in which the object of interest is the predictive distributions defined by the parameters, i.e., $q(\cdot | \vtheta)$.  Since the marginal $p_\vx$ is correct also for the model,  we mainly care about the distribution $q(\cdot|\vx,\vtheta)$ given for various $\vx$s. Thus, to define a distance between two $q(\cdot | \vtheta)$ and $q(\cdot | \vtheta')$ as the mean over $\vx$ of the KL-divergence:
$$
d(q(\cdot|\vtheta),q(\cdot|\vtheta') )\coloneqq \E_{\vx\sim p_\vx}\left[\text{KL}(q(\cdot|\vx,\theta) \parallel q(\cdot|\vx,\theta')\right]
$$
\citet{kim22fisher} showed that for $\veta$ we have 
$$
d(q(\cdot|\vtheta),q(\cdot|\vtheta + \veta) ) \approx \veta^\T \matF(\vtheta) \veta
$$
Motivated by this, we define the metric as  $\metric{\veta}{\vxi}{\vtheta} \coloneqq \veta^\T (\matF(\vtheta) + \matLambda)\vxi$, where we also added the regularization term for Tikhonov regularization. 

The asymptotic behavior of $n\cdot d(q(\cdot|\vtheta^\star),q(\cdot|\hat{\vtheta}) )$ is approximately equal to $ \trace(\matF(\vtheta^\star) \matH_\kappa(\vtheta^\star)^{-1} \matC_\kappa(\vtheta^\star) \matH_\kappa(\vtheta^\star)^{-1})$. If the neural network model is calibrated at $\vtheta^\star$ with noncorrelations of Eqs.~\eqref{eq:uncor_err} holding, and Tikhonov regularization is used, such that $\vtheta^\star$ is bounded as $\matLambda \to 0$, we get
\begin{align*}
 \trace(\matF(\vtheta^\star) \matH_\matLambda(\vtheta^\star)^{-1} \matC_\matLambda(\vtheta^\star) \matH_\matLambda(\vtheta^\star)^{-1})
 &=
 \trace(\matC(\vtheta^\star) \matH_\matLambda(\vtheta^\star)^{-1} \matC_\matLambda(\vtheta^\star) \matH_\matLambda(\vtheta^\star)^{-1})\\
 &= \trace(\matC(\vtheta^\star)\matH_\matLambda(\vtheta^\star)^{-1}[\matC(\vtheta^\star)-\matLambda\vtheta^\star{\vtheta^\star}^\T\matLambda]\matH_\matLambda(\vtheta^\star)^{-1})\\
 &= \trace(\bar{\matH}_\matLambda(\vtheta^\star)^{-2}) - \trace(\matH(\vtheta^\star)\matH_\matLambda(\vtheta^\star)^{-1}\matLambda\vtheta^\star{\vtheta^\star}^\T\matLambda\matH_\matLambda(\vtheta^\star)^{-1})\\
 &= \trace(\bar{\matH}_\matLambda(\vtheta^\star)^{-2}) - {\vtheta^\star}^\T\matLambda\matH_\matLambda(\vtheta^\star)^{-1}\matH(\vtheta^\star)\matH_\matLambda(\vtheta^\star)^{-1}\matLambda\vtheta^\star\\
 &\le
 \trace(\bar{\matH}_\matLambda(\vtheta^\star)^{-2}) - \TnormS{\vtheta^\star}\Tnorm{\matLambda\matH_\matLambda(\vtheta^\star)^{-1}\matH(\vtheta^\star)\matH_\matLambda(\vtheta^\star)^{-1}\matLambda}\\
 &= \trace(\bar{\matH}_\matLambda(\vtheta^\star)^{-2}) + O(\Tnorm{\matLambda})
\end{align*}
The reason for the last equation is that
\[
\Tnorm{\matH_\matLambda(\vtheta^\star)^{-1}\matH}\le 1\quad\text{and}\quad\Tnorm{\matH_\matLambda(\vtheta^\star)^{-1}\matLambda}<1.
\]
Again, we have the trace of the square inverse of the Riemannian Hessian (the proof is similar to the proof of Proposition~\ref{prop:limit_with_o}).

\section{Generalization for Calibrated Models with Tikhonov Regularization}

We now consider the case where Tikhonov regularization is used and the model is an exponential family neural network, with the non-correlations of  Eqs.~\eqref{eq:uncor_err} holding, and is calibrated at $\vtheta^\star$. We show that we can evaluate the gaps between $\hat{\vtheta}$ and $\vtheta^\star$ well using only the soft rank of the Fisher Information Matrix (FIM) $\matF(\vtheta^\star)$. This is important since, as we shall see, the soft rank of the FIM can be estimated well using training data, while other measures in use cannot. %


\begin{proposition}
\label{prop:asymptotic_bounds}
    For a calibrated exponential family neural network model with Tikhonov regularization, if Eqs.~\eqref{eq:uncor_err} hold we have
    \[
    \begin{array}{ll}
    \displaystyle \lim_{n\to\infty} n\cdot\E\left[\TnormS{\hat{\vtheta} - \vtheta^\star} \right] 
    & \leq \trace\left(\left(\matF(\vtheta^\star)+\matLambda\right)^{-1}\right) \\[10pt]
    \displaystyle \lim_{n\to\infty} n\cdot\E\left[\mathrm{Gap}\right] 
    & \leq \rank_\matLambda\left(\matF(\vtheta^\star)\right) \\[10pt]
    \displaystyle \lim_{n\to\infty} n\cdot\E\left[\XnormS{\hat{l}_\kappa -l^\star_\kappa}{L_2(p_\vz)}\right] 
    & \leq \rank^{(2)}_\matLambda\left(\matF(\vtheta^\star)\right)
    \end{array}
    \]
    where $\hat{l}_\kappa \coloneqq l_\kappa(\hat{\vtheta};\cdot)$ and $l^\star_\kappa \coloneqq l_\kappa(\vtheta^\star;\cdot)$.
\end{proposition}
\begin{proof}
\label{app:proof_asymptotic_bounds}
For the first inequality, the calibration assumption with Proposition~\ref{prop:calibrated_model}, Eq.~\eqref{eq:C_Lambda} and Eq.~\eqref{eq:P_Lambda} give: 
\begin{align}
\nonumber
\matH_\matLambda(\vtheta^\star)^{-1}\matC_\matLambda(\vtheta^\star)\matH_\matLambda(\vtheta^\star)^{-1}
\nonumber
&=(\matF(\vtheta^\star)+\matLambda)^{-1}(\matF(\vtheta^\star)-\matLambda \vtheta^\star {\vtheta^\star}^\T \matLambda)(\matF(\vtheta^\star)+\matLambda)^{-1}\\
\nonumber
&\preceq (\matF(\vtheta^\star)+\matLambda)^{-1}\matF(\vtheta^\star)(\matF(\vtheta^\star)+\matLambda)^{-1}\\
&\preceq\  (\matF(\vtheta^\star)+\matLambda)^{-1}.
\end{align}
where the second inequality holds due to Lemma \ref{lemma:ABA}. Now take trace on both sides and recall Eq.~\eqref{eq:theta-error}. 

For the second inequality, 
recall that the exact expression for the limit is $\trace(\matC_\matLambda(\vtheta^\star)\matH_\matLambda(\vtheta^\star)^{-1})$. 
Now,
\begin{align*}    \trace(\matC_\matLambda(\vtheta^\star)\matH_\matLambda(\vtheta^\star)^{-1}) 
    &= \trace((\matF(\vtheta^\star)-\matLambda \vtheta^\star {\vtheta^\star}^\T \matLambda)(\matF(\vtheta^\star)+\matLambda)^{-1})\\
    &\le \trace(\matF(\vtheta^\star)(\matF(\vtheta^\star)+\matLambda)^{-1})\\
    &= \rank_\matLambda(\matF(\vtheta^\star))
\end{align*} 
where the inequality follows from the fact the trace of the product of positive semidefinite matrix and positive definite matrix is always non-negative due to cyclicity of the trace. 

The third inequality follows now from recalling that the exact expression of the limit is $\trace((\matC_\matLambda(\vtheta^\star)\matH_\matLambda(\vtheta^\star)^{-1})^2)$. Using Lemma~\ref{lemma:ABC} in the appendix, we have:
\begin{align*}    \trace((\matC_\matLambda(\vtheta^\star)\matH_\matLambda(\vtheta^\star)^{-1})^2) 
    &= \trace(((\matF(\vtheta^\star)-\matLambda \vtheta^\star {\vtheta^\star}^\T \matLambda)(\matF(\vtheta^\star)+\matLambda)^{-1})^2)\\
    &\le \trace((\matF(\vtheta^\star)(\matF(\vtheta^\star)+\matLambda)^{-1})^2)\\
    &= \rank^{(2)}_\matLambda(\matF(\vtheta^\star))
\end{align*}
\end{proof}

If we further assume that $\Tnorm{\vtheta^\star}$ is bounded  (e.g., if $\matTheta$ is bounded) then we can show that the above inequalities are tight in the following way:
\begin{proposition}
    \label{prop:limit_with_o}
    For a calibrated exponential family neural network model with Tikhonov regularization, if \eqref{eq:uncor_err} holds and $\Tnorm{\vtheta^\star}$ is bounded as $\matLambda\to0$, then we have
    \[
    \begin{array}{ll}
        \lim_{n\to\infty} n\cdot\E\left[\mathrm{Gap}\right] 
        &= \rank_\matLambda(\matF(\vtheta^\star))  + O(\|\matLambda\|_2)\\
        \lim_{n\to\infty} n\cdot\E\left[\XnormS{\hat{l}_\kappa -l^\star_\kappa}{L_2(p_\vz)}\right] 
        &= \rank^{(2)}_\matLambda(\matF(\vtheta^\star)) + O(\|\matLambda\|_2)
    \end{array}
    \]
\end{proposition}
\begin{proof}
\label{app:proof_limit_with_o}
 We  prove the second equality; the first equation can be proved similarly. We have:
    \begin{align*}
            &\lim_{n\to\infty} n\cdot\E\left[\XnormS{\hat{l}_\kappa -l^\star_\kappa}{L_2(p_\vz)}\right]\\ &=\trace((\matC_\matLambda(\vtheta^\star)\matH_\matLambda(\vtheta^\star)^{-1})^2) \\
        &= \trace((\matF(\vtheta^\star)-\matLambda \vtheta^\star {\vtheta^\star}^\top \matLambda)(\matF(\vtheta^\star)+\matLambda)^{-1})^2)\\
        &= \trace((\matP_\matLambda(\matF(\vtheta^\star)) -\matLambda \vtheta^\star {\vtheta^\star}^\T \matP^\perp_\matLambda(\matF(\vtheta^\star)))^2)\\
        &=\rank^{(2)}_\matLambda(\matF(\vtheta^\star))-2\trace(\matP_\matLambda(\matF(\vtheta^\star))\matLambda \vtheta^\star {\vtheta^\star}^\T \matP^\perp_\matLambda(\matF(\vtheta^\star)))+\trace((\matLambda \vtheta^\star {\vtheta^\star}^\T \matP^\perp_\matLambda(\matF(\vtheta^\star)))^2)
    \end{align*}
    The second equation is because the model is a calibrated exponential family neural network and Eqs.~\eqref{eq:uncor_err} hold. Let $M\in\R$ be such that $\TnormS{\vtheta^\star}\leq M$ for sufficiently small $\matLambda$. Since
        $\Tnorm{\matP_\matLambda(\matS)}< 1$ and $\Tnorm{\matP^{\perp}_\matLambda(\matS)} \le 1$, for any positive semidefinite $\matS$ and sufficiently small $\matLambda$ we get:
    \begin{align*}
    \trace(\matP_\matLambda(\matF(\vtheta^\star))\matLambda \vtheta^\star {\vtheta^\star}^\T \matP^\perp_\matLambda(\matF(\vtheta^\star))) &= {\vtheta^\star}^\T \matP^\perp_\matLambda(\matF(\vtheta^\star))\matP_\matLambda(\matF(\vtheta^\star))\matLambda \vtheta^\star \\
    &\leq M \cdot \Tnorm{\matP^\perp_\matLambda(\matF(\vtheta^\star))\matP_\matLambda(\matF(\vtheta^\star))\matLambda}\\
    &\leq M \cdot \Tnorm{\matLambda}
    \end{align*}
    where the first inequality follows from the fact $\matP^\perp_\matLambda(\matF(\vtheta^\star))\matP_\matLambda(\matF(\vtheta^\star))\matLambda$ is a symmetric positive definite matrix. We also have
    \begin{align*}
        \trace((\matLambda \vtheta^\star {\vtheta^\star}^\T \matP^\perp_\matLambda(\matF(\vtheta^\star)))^2) &= 
        \trace(\matLambda \vtheta^\star {\vtheta^\star}^\T \matP^\perp_\matLambda(\matF(\vtheta^\star))\matLambda \vtheta^\star {\vtheta^\star}^\T \matP^\perp_\matLambda(\matF(\vtheta^\star)))\\
        &=  {\vtheta^\star}^\T \matP^\perp_\matLambda(\matF(\vtheta^\star))\matLambda \vtheta^\star \cdot{\vtheta^\star}^\T \matP^\perp_\matLambda(\matF(\vtheta^\star))\matLambda \vtheta^\star \\
        &= {\vtheta^\star}^\T \matLambda (\matF(\vtheta^\star) +\matLambda)^{-1}\matLambda \vtheta^\star \cdot 
        {\vtheta^\star}^\T \matF(\vtheta^\star) (\matF(\vtheta^\star) +\matLambda)^{-1}\matLambda \vtheta^\star\\ 
        & = {\vtheta^\star}^\T \matLambda (\matF(\vtheta^\star) +\matLambda)^{-1}\matLambda \vtheta^\star \cdot( {\vtheta^\star}^\T \matLambda {\vtheta^\star} - {\vtheta^\star}^\T \matLambda (\matF(\vtheta^\star) +\matLambda)^{-1}\matLambda \vtheta^\star) \\
        & \geq - ({\vtheta^\star}^\T \matLambda (\matF(\vtheta^\star) +\matLambda)^{-1}\matLambda \vtheta^\star)^2 \\
        & \geq -M^2 \cdot \TnormS{\matLambda (\matF(\vtheta^\star) +\matLambda)^{-1}\matLambda} \\
        & \geq -M^2 \cdot \TnormS{\matLambda}
    \end{align*}
    Combining the above, we have
    \begin{equation*}
        \rank^{(2)}_\matLambda(\matF(\vtheta^\star))- 2M\Tnorm{\matLambda}-M^2\TnormS{\matLambda} \leq  \lim_{n\to\infty} n\cdot\E\left[\XnormS{\hat{l}_\kappa -l^\star_\kappa}{L_2(p_\vz)}\right] \leq \rank^{(2)}_\matLambda(\matF(\vtheta^\star)) 
    \end{equation*}
    Since we are considering the asymptotics as $\matLambda \to 0$, this establishes the result. 
\end{proof}

Next, we consider what happens if we take $\matLambda\to 0$:
\begin{proposition}
\label{prop:limit_with_lambda}
    Let $\matP_\perp^\star$ be the projection on the null space of $\matF(\vtheta^\star)$. For a calibrated exponential family neural network model, if Eqs.~\eqref{eq:uncor_err} hold, then we have: 
    \begin{equation*}
        \lim_{\matLambda\to 0}\lim_{\substack{n\to\infty }} n\cdot\E\left[\TnormS{\hat{\vtheta} - \vtheta^\star} \right]
        = \trace(\matF(\vtheta^\star)^\pinv) - \|{\matP_\perp^\star}\vtheta^\star\|^2_2 
    \end{equation*}   
    Further assuming that $\Tnorm{\vtheta^\star}$ is bounded, then: 
    \[
    \begin{array}{ll}      
        \lim_{\matLambda\to 0}\lim_{\substack{n\to\infty }}  n\cdot\E\left[\mathrm{Gap}\right] &=\rank(\matF(\vtheta^\star)) \\
        \lim_{\matLambda\to 0}\lim_{\substack{n\to\infty }} n\cdot\E\left[\XnormS{\hat{l}_\kappa -l^\star_\kappa}{L_2(p_\vz)}\right] &=
        \rank(\matF(\vtheta^\star))
    \end{array}
    \]
\end{proposition}
\begin{proof}
For the first equation:
    \begin{align*}    \matH_\matLambda(\vtheta^\star)^{-1}\matC_\matLambda(\vtheta^\star)\matH_\matLambda(\vtheta^\star)^{-1}
    &=(\matF(\vtheta^\star)+\matLambda)^{-1}(\matF(\vtheta^\star)-\matLambda \vtheta^\star {\vtheta^\star}^\T \matLambda)(\matF(\vtheta^\star)+\matLambda)^{-1}\\
    &=(\matF(\vtheta^\star)+\matLambda)^{-1}\matP_\matLambda(\matF(\vtheta^\star))-\matP^\perp_\matLambda(\matF(\vtheta^\star))^\T{\vtheta^\star}{\vtheta^\star}^\T\matP^\perp_\matLambda(\matF(\vtheta^\star)).
    \end{align*}
    Now take trace and $\matLambda \to 0$ to get $\trace(\matF(\vtheta^\star)^\pinv) - \|\matP_\perp^\star\vtheta^\star\|^2_2$ and recall Eq.~\eqref{eq:theta-error}. The two other equations are a direct corollary of Proposition~\ref{prop:limit_with_o}.
\end{proof}

An interesting observation is that the parameter estimation error behaves opposite to generalization: while
\(
\matF \mapsto \rank_\matLambda(\matF)
\) and \(
\matF \mapsto \rank^{(2)}_\matLambda(\matF)
\)
are concave monotone increasing function on the positive semidefinite cone,
$
\matF \mapsto \trace((\matF+\matLambda)^{-1})
$
is a convex monotone decreasing function on the positive semidefinite cone. In other words, the more flat the model is, the higher the expected parameter estimation error is, but lower is the expected prediction error. An implication of this observation is that while in supervised learning we want flat minima, in transfer learning the situation depends a lot on the similarity between the FIMs of the source and target tasks. 

\section{Non Calibrated Models}

When the model is not calibrated at $\vtheta^\star$, then there is no equality between unregularized information matrices, and we do not have an approximation of the gap involving only the FIM. We still have:
\begin{proposition}
    \label{prop:limit_with_o_not_calibrated}
    Assuming that $\Tnorm{\vtheta^\star}$ is bounded as $\matLambda\to 0$, we have
    \begin{equation*}        
        \lim_{n\to\infty} n\cdot\E\left[\mathrm{Gap}\right] 
        = \trace(\matC(\vtheta^\star)\matH_\matLambda(\vtheta^\star)^{-1})  + O(\Tnorm{\matLambda})        
    \end{equation*}
\end{proposition}
\begin{proof}
    The proof is the same as in Proposition~\ref{prop:limit_with_o}, we just can't replace $\matC(\vtheta^\star)$ with $\matF(\vtheta^\star)$ and $\matH(\vtheta^\star)$ with $\matF(\vtheta^\star)$.
\end{proof}

However, it has long been observed that in neural networks the FIM at local minima is a good approximation of the Hessian at that point \citep{martens2020new, kunstner2019limitations, thomas2020interplay}. Since $\vtheta^\star$ is a local minima, this observation applies. As for the gradient covariance,  \citet{thomas2020interplay} empirically observed that in  classification, during training process the gradient covariance $\matC(\vtheta)$ is approximately proportional to the FIM $\matF(\vtheta)$, that is
\[
\matC(\vtheta_t) \approx \alpha_t \matF(\vtheta_t)
\]
where $t$ is an iteration index, and $\alpha_t \in \R$ for $t=0,1,2,\dots$. As a consequence of this approximation, \citet{thomas2020interplay,naganuma2022takeuchis} approximated the asymptotic expected gap as:
\begin{align}
\lim_{n\to\infty} n\cdot \E\left[\mathrm{Gap}\right] &= \trace(\matC(\vtheta^\star)\matH(\vtheta^\star)^{-1}) \nonumber \\ 
&\approx \trace(\matC(\vtheta^\star)\matF(\vtheta^\star)^{-1}) \nonumber \\ 
&\approx 
\frac{\trace(\matC(\vtheta^\star))}{\trace(\matF(\vtheta^\star))}. \label{eq:thomas}
\end{align}
However, this approximation is deficient in two aspects: (1) it does not take regularization into account, and (2) even if we take $\matLambda \to 0$ the approximation should actually be: 
\begin{equation}
\label{eq:correct-approx}
\lim_{\matLambda\to0}\lim_{n\to\infty} n\cdot \E\left[\mathrm{Gap}\right] \approx \frac{\trace(\matC(\vtheta^\star))}{\trace(\matF(\vtheta^\star))}\cdot\rank(\matF(\vtheta^\star)) 
\end{equation}
To see this, first approximate $\matC(\vtheta^\star) \approx (\trace(\matC(\vtheta^\star))/\trace(\matF(\vtheta^\star))) \cdot \matF(\vtheta^\star)$, and by Proposition~\ref{prop:limit_with_o_not_calibrated}: 
\begin{align*}
\lim_{n\to\infty} n\cdot \E\left[\mathrm{Gap}\right] &= \trace(\matC(\vtheta^\star)\matH_\matLambda(\vtheta^\star)^{-1}) + O(\|\matLambda\|_2) \\
&\approx \trace(\matC(\vtheta^\star)(\matF(\vtheta^\star)+\matLambda)^{-1})  + O(\|\matLambda\|_2)  \\
&\approx \frac{\trace(\matC(\vtheta^\star))}{\trace(\matF(\vtheta^\star))} \cdot \rank_\matLambda(\matF(\vtheta^\star)) + O(\Tnorm{\matLambda})
\end{align*}
Taking $\matLambda\to 0$ we obtain Eq.~\eqref{eq:correct-approx}.

In general, we propose to approximate
\begin{equation}
   \lim_{n\to\infty} n\cdot \E\left[\mathrm{Gap}\right] \approx  \frac{\trace(\matC(\vtheta^\star))}{\trace(\matF(\vtheta^\star))} \cdot \rank_\matLambda(\matF(\vtheta^\star)) \label{eq:our-gap}
\end{equation}

\begin{remark}
    Due to numerical stability considerations, \citet{thomas2020interplay} proposes to use $\tilde{\matF}(\vtheta^\star)$ in lieu of $\matF(\vtheta^\star)$ in Eq.~\eqref{eq:thomas}, where $\tilde{\matF}(\vtheta^\star)$ is obtained from $\matF(\vtheta^\star)$ by truncating singular values which are less than $0.001\cdot\Tnorm{\matF(\vtheta^\star)}$. Eq.~\eqref{eq:our-gap} is a measure based on the actual weight decay $\matLambda$ used, instead of the heuristic $0.001$. 
\end{remark}

Figure~\ref{fig:Gap-estimation-using} shows the state-of-the-art result using Eq. \eqref{eq:our-gap}  compares to naive
implementation of regularized TIC and \citet{thomas2020interplay}'s approach.

{%
\setlength{\overfullrule}{0pt}%
\begin{figure*}
    \centering
    \hspace{-2cm}%
    \begin{minipage}[t]{0.15\textwidth}%
        \includegraphics[scale=0.4]{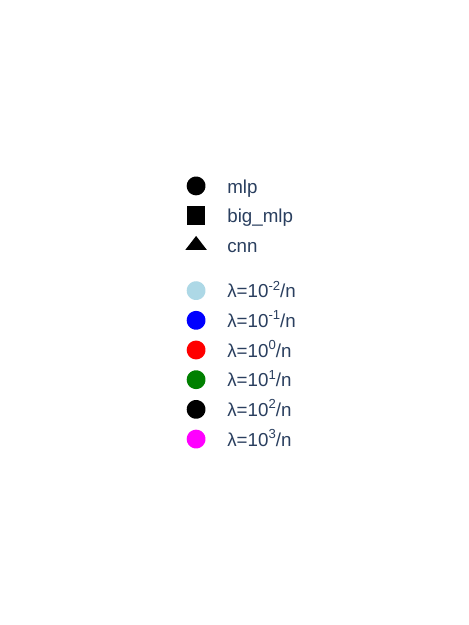}%
    \end{minipage}%
    \hspace{-0.2cm}%
    \begin{minipage}[t]{0.26\textwidth}%
        \includegraphics[scale=0.49]{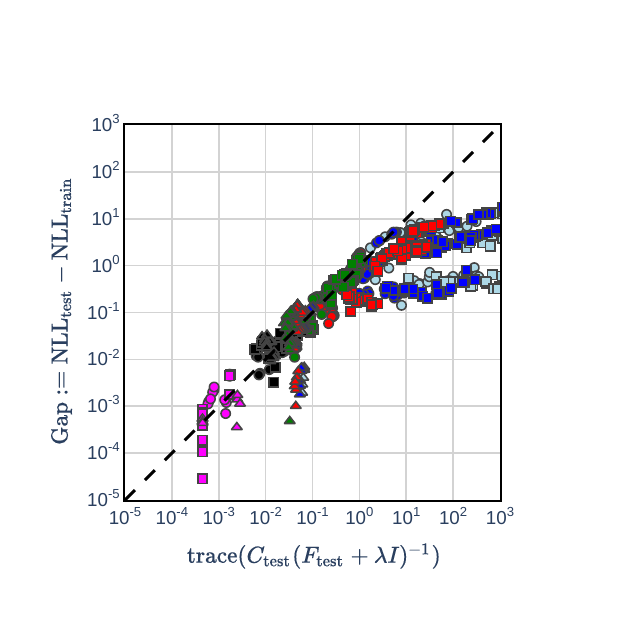}%
    \end{minipage}%
    \hspace{0.6cm}%
    \begin{minipage}[t]{0.26\textwidth}%
        \includegraphics[scale=0.49]{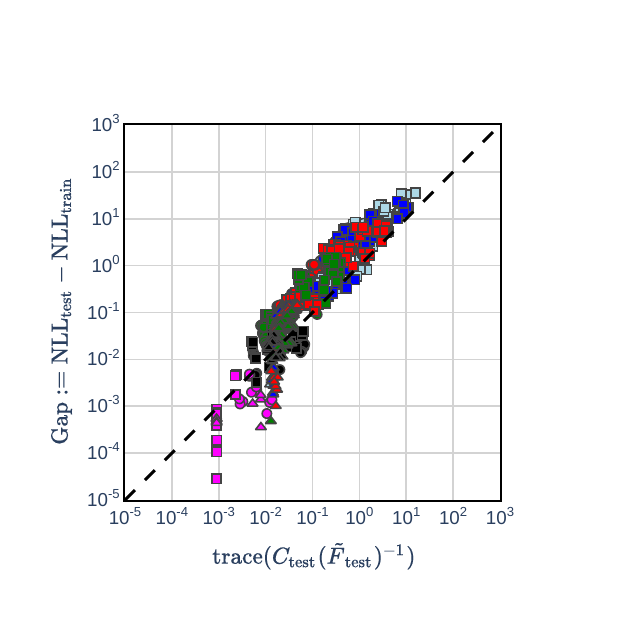}%
    \end{minipage}%
    \hspace{0.6cm}%
    \begin{minipage}[t]{0.26\textwidth}%
        \includegraphics[scale=0.49]{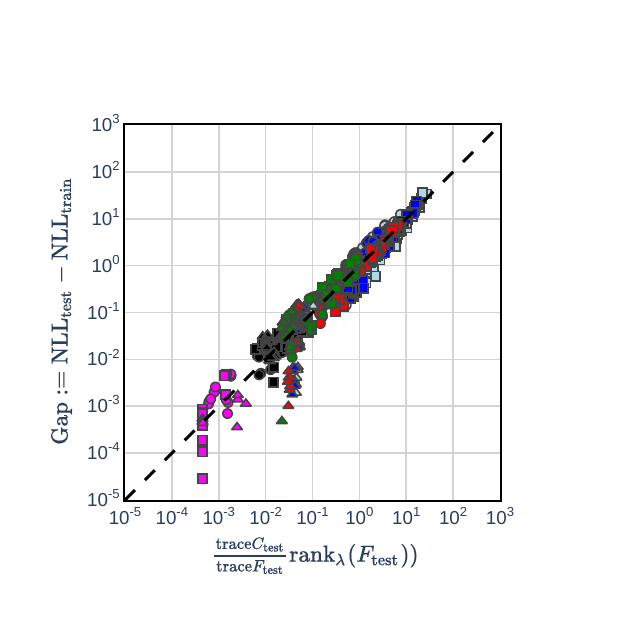}%
    \end{minipage}%
\caption{Gap estimation with $\matC(\hat{\vtheta})$
and $\matF(\hat{\vtheta})$ computed on the test data. The leftmost plot shows the regularized TIC; the middle plot shows the gap estimator based on \citet{thomas2020interplay}; and the rightmost plot shows our approach.}
\label{fig:Gap-estimation-using}
\end{figure*}

\subsection{On The Relation Between \texorpdfstring{$\matC$}{C} and \texorpdfstring{$\matF$}{F}}
\label{sec:on_The_relation}
As mentioned above, \citet{thomas2020interplay} empirically observed that during training $\matC(\vtheta)$ is proportional to $\matF(\vtheta)$. In this section, we shed additional light on this phenomenon in classification. Throughout this section we assume that $\vy$ is a one-hot encoding of the class, $\netN_\vtheta(\vx)$ is a vector of class probabilities (so it has the same dimensions as $\vy$), and $q(\cdot|\netN_\vtheta(\vx))$ is the categorical distribution.

Recall the symbols defined by Eq.~\eqref{eq:e_theta} and let 
\[
\matB_\vtheta = \ve_\vtheta \ve_\vtheta^\T
\]
Our experiments further refined \citet{thomas2020interplay}'s observation. 
We found (empirically) that (see Section~\ref{app:experimental_results_CF}) $\matJ_{\hat{\vtheta}}$ is uncorrelated with $\matB_{\hat{\vtheta}}$ and $\matSigma_{\hat{\vtheta}}$ in the sense that 
\begin{align*}
   \matC(\hat{\vtheta}) &= \E_\vz\!\left[\matJ_{\hat{\vtheta}}^\T \matB_{\hat{\vtheta}}\matJ_{\hat{\vtheta}}\right] 
   \;\approx\; \E_\vz\!\left[\matJ_{\hat{\vtheta}}^\T \,\E_\vz\!\left[\matB_{\hat{\vtheta}}\right]\matJ_{\hat{\vtheta}}\right], \\
   \matF(\hat{\vtheta}) &= \E_\vz\!\left[\matJ_{\hat{\vtheta}}^\T \matSigma_{\hat{\vtheta}}\matJ_{\hat{\vtheta}}\right] 
   \;\approx\; \E_\vz\!\left[\matJ_{\hat{\vtheta}}^\T \,\E_\vz\!\left[\matSigma_{\hat{\vtheta}}\right]\matJ_{\hat{\vtheta}}\right],
\end{align*}
and that $\E\!\left[\matB_{\hat{\vtheta}}\right]$ is proportional to $\E\!\left[\matSigma_{\hat{\vtheta}}\right]$, i.e., $\E\!\left[\matB_{\hat{\vtheta}}\right] \approx \alpha \,\E\!\left[\matSigma_{\hat{\vtheta}}\right]$ where $\alpha$ is the same $\alpha$ for which $\matC(\hat{\vtheta}) \approx \alpha \,\matF(\hat{\vtheta})$. $\matC(\hat{\vtheta}) \approx \alpha \,\matF(\hat{\vtheta})$. 

Consider the situation at $\vtheta^\star$ (instead of $\hat{\vtheta}$). If both observations hold exactly at $\vtheta^\star$, then $\matC(\vtheta^\star)$ would have been exactly proportional to $\matF(\vtheta^\star)$, and our proposed estimator is a good estimator of the gap. The following proposition shows that even if only the first observation holds, a slightly weaker bound on the asymptotic gap holds.
\begin{proposition}
\label{prop:spectral_gap_estimate}
    Suppose Eq.~\eqref{eq:uncor_err}
holds at $\vtheta^\star$ for $\matC(\vtheta^\star)$ and $\matF(\vtheta^\star)$, and suppose $\Tnorm{\vtheta^\star}$ is bounded as $\matLambda\to0$. Denote $\alpha^\star=\trace(\E[\matB(\vtheta^\star)])/\trace(\E[\matSigma(\vtheta^\star)])$, and let $\kappa$ be the condition number of $\E\left[\matJ(\vtheta^\star)(\matF(\vtheta^\star)+\matLambda)^{-1}\matJ(\vtheta^\star)^\T \right]$. Then, 
    $$
    \kappa^{-1}\alpha^\star \leq \frac{\lim_{n\to\infty}n\cdot\E\left[\mathrm{Gap}\right]+O(\Tnorm{\matLambda})}{\rank_\matLambda(\matF(\vtheta^\star))} \leq \kappa \alpha^\star 
    $$
\end{proposition}
\begin{proof}
\label{app:proof_spectral_gap_estimate}
Eq.~\eqref{eq:uncor_err} holds for $\matC(\vtheta^\star)$ and $\matF(\vtheta^\star)$ at $\vtheta^\star$, we have:
    \begin{align*}
        \trace(\matC(\vtheta^\star)(\matF(\vtheta^\star)+\matLambda)^{-1})&=\trace(\E[\matJ(\vtheta^\star)^\T\E[\matB(\vtheta^\star)]\matJ(\vtheta^\star)](\matF(\vtheta^\star)+\matLambda)^{-1}) \\
        &= \trace(\E[\matB(\vtheta^\star)]\E[\matJ(\vtheta^\star)(\matF(\vtheta^\star)+\matLambda)^{-1}\matJ(\vtheta^\star)^\T])
    \end{align*}
    and:
    \begin{align*}
        \trace(\matF(\vtheta^\star)(\matF(\vtheta^\star)+\matLambda)^{-1})&=\trace(\E[\matJ(\vtheta^\star)^\T\E[\matSigma(\vtheta^\star)]\matJ(\vtheta^\star)](\matF(\vtheta^\star)+\matLambda)^{-1}) \\
        &= \trace(\E[\matSigma(\vtheta^\star)]\E[\matJ(\vtheta^\star)(\matF(\vtheta^\star)+\matLambda)^{-1}\matJ(\vtheta^\star)^\T])
    \end{align*}
    Thus:
    \begin{equation*}
    \lambda_{\min} \cdot \trace(\E[\matB(\vtheta^\star)]) \leq \trace(\matC(\vtheta^\star)(\matF(\vtheta^\star)+\matLambda)^{-1}) \leq \lambda_{\max}\cdot \trace(\E[\matB(\vtheta^\star)])
    \end{equation*}
    and:
    \begin{equation*}
        \lambda_{\min} \cdot \trace(\E[\matSigma(\vtheta^\star)]) \leq \trace(\matF(\vtheta^\star)(\matF(\vtheta^\star)+\matLambda)^{-1}) \leq \lambda_{\max}\cdot \trace(\E[\matSigma(\vtheta^\star)])
   \end{equation*}
   where $\lambda_{\min}$ and $\lambda_{\max}$ are the extreme eigenvalues of $\E\left[\matJ(\vtheta^\star)(\matF(\vtheta^\star)+\matLambda)^{-1}\matJ(\vtheta^\star)^\T \right]$.
    Write:
    \begin{align*}
        \lim_{n\to\infty}n\cdot\E\left[\mathrm{Gap}\right] 
        &= \trace(\matC(\vtheta^\star)(\matF(\vtheta^\star)+\matLambda)^{-1}) + O(\Tnorm{\matLambda})\\
        &= \frac{\trace(\matC(\vtheta^\star)(\matF(\vtheta^\star)+\matLambda)^{-1})}{\trace(\matF(\vtheta^\star)(\matF(\vtheta^\star)+\matLambda)^{-1})}\rank_\matLambda(\matF(\vtheta^\star)) + O(\Tnorm{\matLambda})
    \end{align*}
    From the inequalities above and recalling that $\alpha^\star=\trace(\E[\matB(\vtheta^\star)])/\trace(\E[\matSigma(\vtheta^\star)])$, we get the desired result.
\end{proof}

The quantity $\alpha^\star$ can be interpreted as a ratio between error and uncertainty.  To see this, note that
\begin{align*}
    \mathrm{MSE}(\vtheta^\star) & \coloneqq \E\left[\TnormS{\vy - \hat{\vy}(\vtheta^\star)}\right] = \E[\trace(\matB(\vtheta^\star))] \\
    \mathrm{Uncertainty}(\vtheta^\star) & \coloneqq \E\left[1 - \TnormS{\hat{\vy}(\vtheta^\star)} \right]\\
    & = \E\left[\trace(\diag(\hat{\vy}(\vtheta^\star)) - \hat{\vy}(\vtheta^\star) \hat{\vy}(\vtheta^\star)^\T)\right] \\
    & = \trace(\E[\matSigma(\vtheta^\star)])
\end{align*}
For a vector of probabilities $\hat{\vy}$ in a multinomial distribution, the more spiked the probabilities are (and so the outcome is more certain), the closer $1-\TnormS{\hat{\vy}}$ will be to zero, and the more equal the values are (and so the outcome is less certain) the larger this value is. Thus,  $\mathrm{Uncertainty}(\vtheta^\star)$ is indeed a measure of mean uncertainty. 

\subsection{Estimating the Gap from Training Data}

Eq.~\eqref{eq:our-gap} provides a reasonable estimate of the gap. However, in the form presented, it cannot be computed from the training data. First, we do not know $\vtheta^\star$. Here, we might attempt to use $\hat{\vtheta}$ of $\vtheta^\star$. However, we now face another issue: our samples from $\vz$ are correlated with $\hat{\vtheta}$, so technically they cannot be used to estimate $\matF(\hat{\vtheta})$ and $\matC(\hat{\vtheta})$. However, experiments reveal that estimating $\matF(\hat{\vtheta})$ using training data and using this to estimate $\rank_\matLambda(\matF(\hat{\vtheta}))$ provides a reasonable estimate for the purpose of estimating the gap. 

\begin{figure}[t]
    \centering
    \includegraphics[scale=0.55]{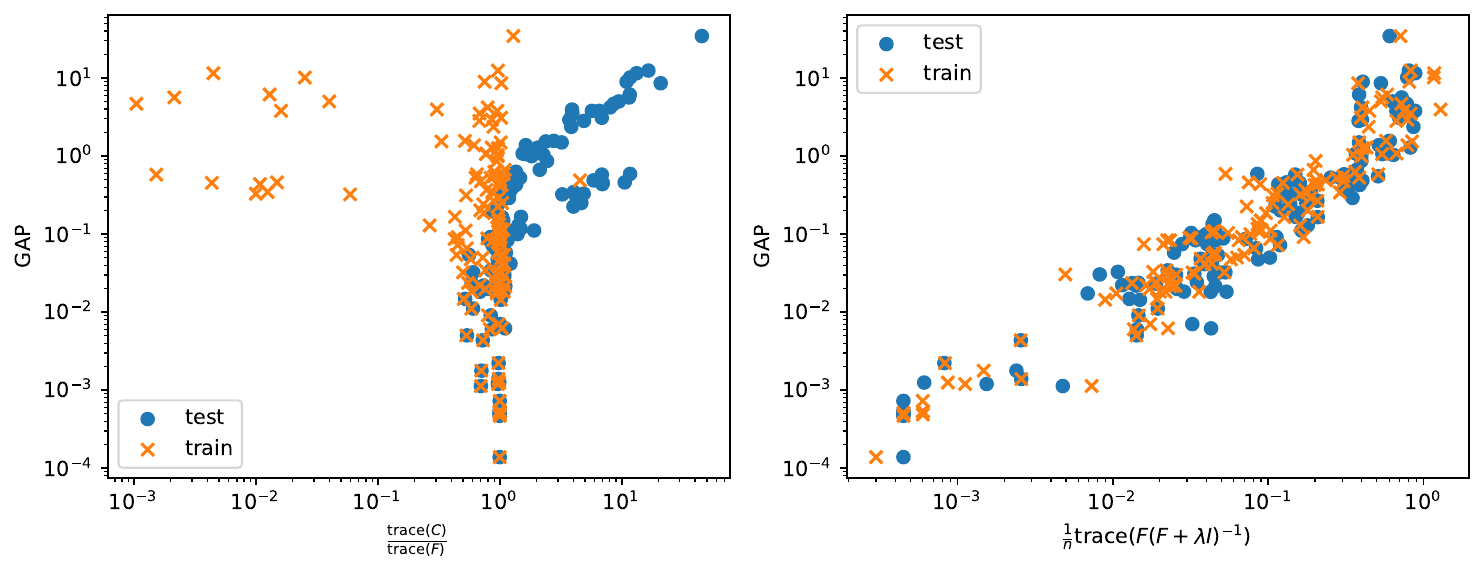}
    \caption{Predicting generalization gap factors from training data. Left: Gap versus $\frac{\trace(\matC)}{\trace(\matF)}$ shows poor correlation due to overfitting; Right: Gap versus $\rank_\matLambda(\matF)$ demonstrates strong correlation despite overfitting (orange: training, blue: test).
    }
    \label{fig:gap_vs_metrics}
\end{figure}

However, the prefactor $\trace(\matC(\hat{\vtheta})) / \trace(\matF(\hat{\vtheta}))$ cannot be reliably estimated from training data using standard approaches, as it is very sensitive to overfitting. See Figure~\ref{fig:gap_vs_metrics}.

To shed light on why this happens, notice that the prefactor is essentially a ratio between error and uncertainty, as discussed in the previous subsection. We have the following:
\begin{proposition}
\label{prop:limit_mse_uncertainty}
    Let $0<\epsilon < 1/2$ and let $i$ be the random integer such that $\vy=\mathrm{onehot}(i)$. Let $\vtheta$ be some {\em fixed} parameters (not a random variable). Assume that $1-\hat{\vy}(\vtheta)_i \leq \epsilon$ almost surely, where $\hat{\vy}(\vtheta)_i$ denotes the $i$th entry to $\hat{\vy}(\vtheta)$. Then, $\frac{\mathrm{MSE}(\vtheta)}{\mathrm{Uncertainty}(\vtheta)}\leq2\epsilon$ almost surely.
\end{proposition}
\begin{proof}
\label{app:proof_limit_mse_uncertainty}
    Since $\vtheta$ is fixed, for conciseness we denote $\hat{\vy}=\hat{\vy}(\vtheta)$. Note that this is still a random variable, since it depends on $\vz$. All expectations in the proof are with respect to $\vz$. 
    
    Let $E:=1-\hat{\vy}_i$. Since $\hat{\vy}_i=1-E$ and $\sum_{j}\hat{\vy}_j=1$, we get  $\sum_{j\ne i}\hat{\vy}_j=E$. Combining the fact that $\hat{\vy}$ is entrywise non-negative, we get $\sum_{j\ne i}\hat{\vy}_j\le E^2$.
    Thus: 
    $$
    \TnormS{\vy-\hat{\vy}} = (1-\hat{\vy}_i)^2 + \sum_{j\ne i}\hat{\vy}^2_j \leq 2E^2
    $$ 
    and, 
    $$
    1-\TnormS{\hat{\vy}} = 1 - \left(\hat{\vy}^2_i +  \sum_{j\ne i}\hat{\vy}^2_j\right) \ge 1-\left((1-E)^2+E^2\right) =2E(1-E).
    $$ Finally, we have
    \begin{equation*}
        \frac{\mathrm{MSE}(\vtheta}{\mathrm{Uncertainty}(\vtheta)} 
        = \frac{\E\left[\TnormS{\vy-\hat{\vy}}\right]}{\E\left[1-\TnormS{\hat{\vy}}\right]}
        \leq \frac{\E[E^2]}{\E[E(1-E)]}
        \leq \frac{\epsilon\E[E]}{(1-\epsilon)\E[E]} 
        = \frac{\epsilon}{1-\epsilon}\leq 2\epsilon.
    \end{equation*}
    
\end{proof}

The last proposition shows that when overfitting, the error-uncertainty ratio on the training set goes to zero. On the other hand, on the test population, in case of overfitting we can expect the MSE to be bounded away from zero, while the uncertainty still goes to zero. So, in the case of overfitting, the error-uncertainty ratio behaves oppositely for training and test sets. We draw two conclusions:
\begin{itemize}
    \item We can only reliably estimate the part of the generalization gap that is due to the soft-rank of $\matF(\vtheta^\star)$ (which is a measure of flatness). 
      
    \item Using this estimate of the gap becomes less and less indicative of the actual gap as the model overfits. However, it is very indicative of the actual gap as long as there is no overfitting, i.e., the error-uncertainty ratio is bounded away from zero. One can hope to gauge if this the case from training data.   
\end{itemize}

\section{Efficient Gap Estimation}
\label{sec:efficient}
Regardless of whether indicative of the actual gap or not, estimating the prefactor $\trace(\matC(\hat{\vtheta})) / \trace(\matF(\hat{\vtheta}))$ is relatively cheap as it involves only the trace of the matrices. Indeed, \citet{thomas2020interplay} used it as a tractable estimate of the gap. However, the term $\rank_\matLambda(\matF(\hat{\vtheta}))$ is much costlier to estimate since it involves full matrices. 

To reduce cost, we propose the use of the {\em Kronecker-Factored Approximate Curvature} (KFAC) \citep{martens2015optimizing,martens2020new} approximation of $\matF(\hat{\vtheta})$, or its diagonal. We then use the soft rank of the approximation as an estimate of $\rank_\matLambda(\matF(\hat{\vtheta}))$.
KFAC first approximates $\matF(\hat{\vtheta})$ using a block diagonal matrix in which the non-zero diagonal blocks coincide with those of $\matF(\hat{\vtheta})$. Then, each diagonal block is approximated as a Kronecker product of two matrices.

Empirically, this scheme yields excellent estimates of the gap (when the prefactor is computed using test data); see Figure~\ref{fig:gap_estimation_approximation}. Theoretically, the following proposition establishes that the first phase (keeping only diagonal blocks) does not make the approximation more optimistic. A weaker version, which shows that (regular) rank  only increases by the approximation (i.e., when $\matB$ is the identity matrix), is an immediate corollary of Theorem 1 in \citep{lundquist-1996-rank-inequal}.
\begin{proposition}
    \label{prop:block_submatrix}
    Let $\matA,\matB\in\R^{n\times n}$ be two positive semidefinite matrices. Suppose we can write 
    $$
    \matA = \begin{bmatrix}
        \matA_{11} &  \cdots & \matA_{1m}\\
        \vdots     & \ddots  & \vdots    \\
        \matA_{m1} & \cdots  & \matA_{mm} 
    \end{bmatrix},\quad 
    \matB = \begin{bmatrix}
        \matB_{11} &         & \\
                   & \ddots  &   \\
                   &         & \matB_{mm} 
    \end{bmatrix}
    $$
    where the diagonal blocks are square and same sizes across $\matA$ and $\matB$. Let, 
    $$
    \tilde{\matA} = \begin{bmatrix}
        \matA_{11} &         & \\
                   & \ddots  &   \\
                   &         & \matA_{mm} 
    \end{bmatrix}
    $$
    Then, $\rank_{\matB}(\matA)\leq \rank_{\matB}(\tilde{\matA})$.
\end{proposition}
(The proof that we give here requires the concept of a concave trace function and Theorem~\ref{theorem:strong}. Both are introduced in the next subsections.)
\begin{proof}
Let $h:[0,\infty)\to[0,1)$, $h(x)=1-(x+1)^{-1}$, and let $\lambda_d\ge\dots\ge\lambda_1$ be the eigenvalues of $\matB^{-\frac{1}{2}} \matA \matB^{-\frac{1}{2}}$ then
\begin{equation*}
\rank_\matB(\matA)
= \trace\left(\matI - \matB\left(\matA+\matB\right)^{-1}\right)
= \trace\left(\matI - \left(\matB^{-\frac{1}{2}}\matA\matB^{-\frac{1}{2}} + \matI\right)^{-1}\right)
= \sum_{i=1}^d h\left(\lambda_i\right).
\end{equation*}
Thus, since a $h$ is continuous and concave, we complete the proof by applying Theorem~\ref{theorem:strong}
\end{proof}
}

\subsection{Concave Trace Functions}
\label{sec:concave_operator}
This section presents two theorems that yield computationally efficient bounds on the soft rank by exploiting its representation as a concave trace function. The same bounds also apply to the log-determinant complexity term arising in the Bayesian formulation.

\begin{definition}[Trace function]\label{definition:trace_function}
Let $h:\R\to\R$ be a scalar function. Its spectral extension to symmetric matrices is
\[
h:\sym{d}\to\R,\qquad
\matX \longmapsto h(\matX)\coloneqq \sum_{i=1}^{d} h\!\left(\lambda_i(\matX)\right),
\]
where $\sym{d}$ denotes the space of real $d\times d$ symmetric matrices and
$\lambda_1(\matX)\ge \lambda_2(\matX)\ge \cdots \ge \lambda_d(\matX)$
are the eigenvalues of $\matX$ ordered nonincreasingly.
\end{definition}

\subsection{Block Diagonal Bound fo Concave Trace Functions}
\begin{definition}[Pinching]
    We say that a matrix $\matB$ is a {\em block diagonal submatrix} of another matrix $\matA$ if we can write 
    $$
    \matA = \begin{bmatrix}
        \matA_{11} &  \cdots & \matA_{1m}\\
        \vdots     & \ddots  & \vdots    \\
        \matA_{m1} & \cdots  & \matA_{mm} 
    \end{bmatrix},\quad 
    \matB = \begin{bmatrix}
        \matA_{11} &         & \\
                   & \ddots  &   \\
                   &         & \matA_{mm} 
    \end{bmatrix}
    $$
    The operator $\calP:\sym{d}\to\sym{d},$ such that $\calP(\matA)=\matB$ is called pinching and can be written as $\calP(\matA)=\sum_{i=1}^m\matD_i\matA\matD_i$ with diagonal projections $\matD_i$.
\end{definition}
\begin{theorem}
\label{theorem:strong}Let a continuous and concave $h:\R\to\R$ define a trace function as in Definition~\ref{definition:trace_function} and let $\calP$ be a pinching operator on $\sym{d}$. For all
$
\matX,\matA \in\sym{d},
$ 
such that $\matA=\calP(\matA)$ is a block diagonal matrix we have,
$$
h\left(\matA\matX\matA\right) \le h\left(\matA\calP(\matX)\matA\right).
$$
\end{theorem}
\begin{proof}
Let $\matY \in\sym{d}$.  By an extension of Schur's theorem, the eigenvalues of $\matY$, $\alpha_{d}\ge\dots\ge\alpha_{1}$,
majorize the eigenvalues of $\calP(\matY)$, $\beta_{d}\ge\dots\ge\beta_{1}$ (for the 2-by-2 block case, see
\citep[Chapter 9, Theorem C.1]{Marshall2010Matrix}; the general case can be obtained by induction). That is, $\sum_{i=1}^{k}\alpha_{i}\ge\sum_{i=1}^{k}\beta_{i}$
for all $1\le k<d$ and $\sum_{i=1}^{d}\alpha_{i}=\sum_{i=1}^{d}\beta_{i}$.
Thus, by Karamata's inequality \citep{karamata1932inegalite,Marshall2010Matrix},
for any concave function $h$, we have
\begin{equation}
\label{eq:weak}
    h\left(\matY\right)=\sum_{i=1}^{d}h\left(\alpha_{i}\right)\le\sum_{i=1}^{d}h\left(\beta_{i}\right)=h\left(\calP(\matY)\right)
\end{equation}
Now, Let $\matA\in \sym{d}$ such that 
$$
\matA=\calP({\matA})=\sum_{i=1}^m \matD_i\matA\matD_i,
$$ we have that
$$
\matA\matD_i=\matD_i\matA\matD_i=\matD_i\matA.
$$
Thus,
\[
\calP(\matA\matX\matA)=\sum_{i=1}^m\matD_i\matA\matX\matA\matD_i=\matA\left(\sum_{i=1}^m\matD_i\matX\matD_i\right)\matA=\matA\calP(\matX)\matA
\]
Finally, let $\matX\in\sym{d}$, substitute $\matY=\matA\matX\matA$ in \eqref{eq:weak} and obtain:
\[
h(\matA\matX\matA) \le h\left(\calP(\matA\matX\matA)\right)=h\left(\matA\calP(\matX)\matA\right).
\]
\end{proof}

\subsection{Batching of Concave Trace Functions}
Sometimes one may wish to optimize the complexity term. Efficient optimization requires that the objective be optimized using stochastic batches. One option is to use {\em Stochastic Compositional Gradient Descent} (SCGD) \citep{wang2017stochastic}. However, if the additional burden of this approach is undesirable, the following theorem may be applied. It is applicable to both the Bayesian complexity and the soft rank, as explained at the beginning of this section.:
\begin{theorem}
\label{theorem:batch}
Let a continuous and concave $h:\R\to\R$  satisfying $h(0)=0$ define a trace function as in Definition~\ref{definition:trace_function}. Then, for all
$\{\matX_i\}_{i=1}^{m}\subset \psd d$, we have 
$$
h\left(\sum_{i=1}^{m}\matX_i\right)
\le \sum_{i=1}^{m}h\left(\matX_i\right).
$$
\end{theorem}
\begin{proof}
Define
$$
\matV \coloneq \begin{bmatrix}
\matX_{1}^{\frac{1}{2}} \\
\matX_{2}^{\frac{1}{2}} \\
\vdots \\
\matX_{m}^{\frac{1}{2}}
\end{bmatrix}\,,
$$
so that
$$
\matV^\T\matV = \sum_{i=1}^{m}\matX_i\,.
$$
The nonzero eigenvalues of $\matV\matV^\T$ coincide with those of $\matV^\T\matV$. Any additional eigenvalues of $\matV\matV^\T$ are zero, and since $h(0)=0$, we obtain
$$
h\bigl(\matV^\T\matV\bigr)
=h\bigl(\matV\matV^\T\bigr).
$$
Now, let
$$
\matY \coloneq \diag\Bigl(\matX_1,\matX_2,\dots,\matX_m\Bigr)
$$
be the block diagonal matrix extracted from $\matV\matV^\T$. By Theorem~\ref{theorem:strong}, the concavity of $h$ implies
$$
h\bigl(\matV\matV^\T\bigr)
\le h\bigl(\matY\bigr)
=\sum_{i=1}^{m}h\bigl(\matX_i\bigr).
$$
Thus, we have
$$
h\Bigl(\sum_{i=1}^{m}\matX_i\Bigr)
=h\bigl(\matV^\T\matV\bigr)
=h\bigl(\matV\matV^\T\bigr)
\le \sum_{i=1}^{m}h\bigl(\matX_i\bigr),
$$
which completes the proof.
\end{proof}
Theorem~\ref{theorem:batch} extends the result of \citet{immer2023stochastic}, which corresponds to the special case $h(x)=\log(x + 1)$.

{%
\setlength{\overfullrule}{0pt}%
\begin{figure*}[t]
    \hspace{-1cm}%
    \begin{minipage}[t]{0.15\textwidth}
        \includegraphics[scale=0.4]{plots/legend_train_all}
    \end{minipage}%
    \hspace{-0.2cm}
    \begin{minipage}[t]{0.26\textwidth}
    \includegraphics[scale=0.49]{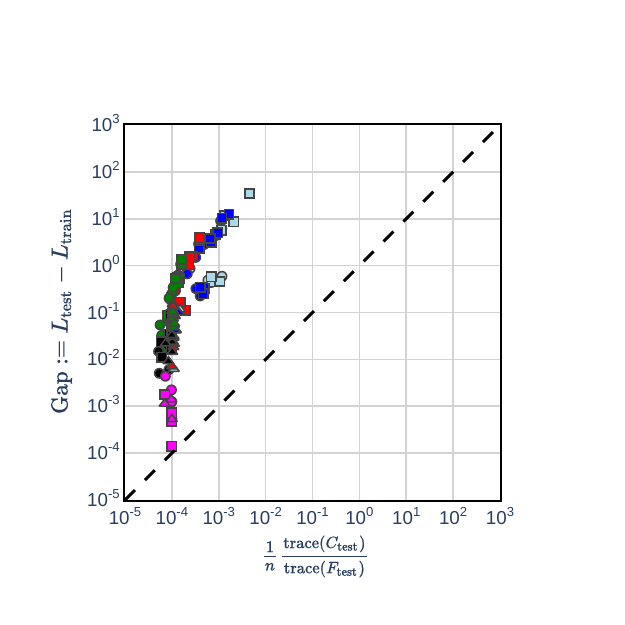}
    \end{minipage}%
    \hspace{0.8cm}
    \begin{minipage}[t]{0.26\textwidth}
        \includegraphics[scale=0.49]{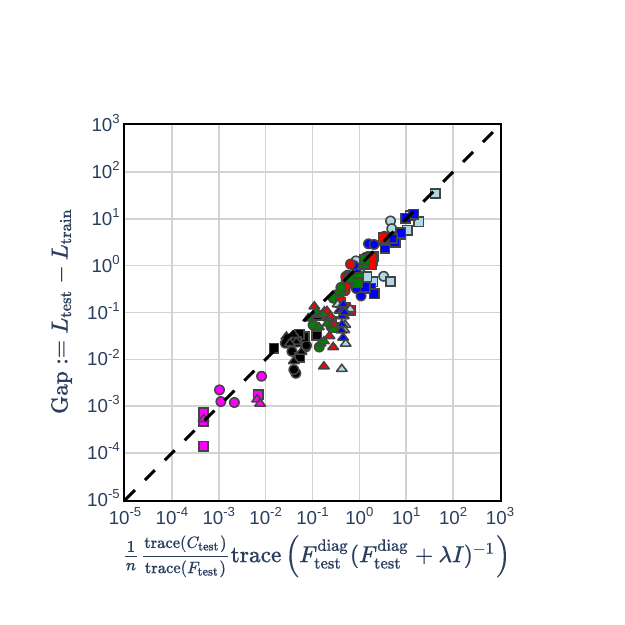}
    \end{minipage}%
    \hspace{0.8cm}
    \begin{minipage}[t]{0.26\textwidth}
        \includegraphics[scale=0.49]{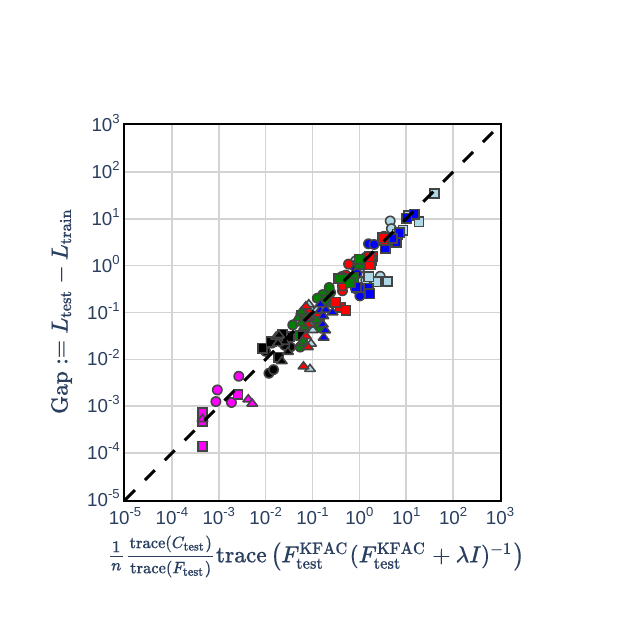}
    \end{minipage}
    \caption{Cheap generalization gap estimation. Left, the trace ratio suggested
    by \citet{thomas2020interplay} severely underestimates the gap
    as it lacks the rank factor. Middle, a diagonal approximation
    of $\matF(\hat{\vtheta})$ is substituted in
    the soft rank-based gap estimation, clearly overestimates the
    gap as the theory predicts. Right, soft rank-based metric with KFAC approximation of $\matF(\hat{\vtheta})$, smaller gap overestimation.}
    \label{fig:gap_estimation_approximation}
\end{figure*}
}

\section{Additional Experiments}
\label{sec:experiments}
We investigate the effectiveness of Hessian-based complexity measures computed from training data alone to predict generalization gaps. Our experimental setup follows \citet{thomas2020interplay}; we evaluate on MNIST \citep{lecun1998mnist}, CIFAR-10 \citep{krizhevsky2009learning}, and SVHN \citep{netzer2011reading}. All images are first downsampled to $7 \times 7$ pixels. Models are trained for 200K steps. 

To compare similarities and discrepancies among Hessian-based complexity measures, we vary the following:
\begin{itemize}
\item \textbf{Architectures:} MLP with one hidden layer (70 ReLU units; 10-way output), BigMLP with two hidden layers (same layer sizes as the MLP), and a 3-layer CNN (first: $3 \times 3$ kernel, no padding, 15 channels; second: $3 \times 3$ kernel, no padding, 20 channels; output: 10 channels for class scores).
\item \textbf{Training-set sizes:} 5k, 10k, and 20k.
\item \textbf{Weight decay:} $\lambda=\frac{c}{n}$ with $c \in \{10^{-2}, 10^{-1}, 1, 10^{1}, 10^{2}, 10^{4}, 10^{5}\}$ and $n$ the training-set size.
\item \textbf{Learning rates:} $0.01$, $0.005$, $0.001$; Adam optimizer ($\mu=0.9$).
\item \textbf{Batch sizes:} $64$, $256$, $512$.
\end{itemize}
Results related to efficient computation (Section \ref{sec:efficient}) were obtained using a subset of parameters (learning rate 0.01, batch size 512).

All experiments were carried out on a server with Intel Xeon E5-2683 v4 CPUs (2.10\,GHz, 64 cores), 755\,GB RAM, and 7 NVIDIA TITAN Xp GPUs (12\,GB each). The full grid required approximately 5 days to run.

\subsection{Correlation, Complexity vs Gap}
Our goal here is to assess the correlations between Hessian-based complexity measures (approximated via FIM) computed from training data and the generalization gap. We measure Kendall rank correlation for: $\rank_{\lambda}(\matF(\hat{\vtheta}))$, $\frac{1}{2}\log\det(\frac{1}{\lambda}\matF(\hat{\vtheta}) + \matI)$, $\trace(\matF(\hat{\vtheta}))$, and $\Tnorm{\matF(\hat{\vtheta})}$ (maximum eigen value). The results, including diagonal and KFAC approximations for soft rank and logarithmic determinant metrics, are shown in Table~\ref{tab:kendall}.

\begin{figure}[htbp]
  \centering
  \captionsetup{type=table}          
  \caption{Kendall $\tau$ correlations between
    training-set-complexity measures and the generalization gap}
  \label{tab:kendall}

  \begin{tabular}{@{}clc@{}}
    \toprule
    & Complexity measure & $\tau$ \\ \midrule
    & $\rank_\lambda(\matF)$                                    & 0.84 \\
    & $\tfrac12\log\det\!\bigl(\lambda^{-1}\matF + \matI\bigr)$ & 0.82 \\
    & $\rank_\lambda(\matF_{\mathrm{KFAC}})$                    & 0.78 \\
    & $\tfrac12\log\det\!\bigl(\lambda^{-1}\matF_{\mathrm{KFAC}} + \matI\bigr)$ & 0.76 \\
    & $\tfrac12\log\det\!\bigl(\lambda^{-1}\matF_{\mathrm{diag}} + \matI\bigr)$ & 0.72 \\
    & $\rank_\lambda(\matF_{\mathrm{diag}})$                    & 0.68 \\
    & $\trace(\matF)$                                           & 0.28 \\
    & $\lVert \matF\rVert_{2}$                                  & 0.08 \\ \bottomrule
  \end{tabular}
\end{figure}




Results show that trace and operator norm metrics are significantly less effective than log determinant and soft rank, even with diagonal FIM approximations. Log determinant and soft rank perform similarly, as expected from their shared properties (strict concavity and monotonic increase).

\subsection{Influence of Overfitting and Calibration}
\label{subsec:ratio_influence}
Next, we examine the influence of $\trace(\matC(\hat{\vtheta)})/\trace(\matF(\hat{\vtheta)})$ on the performance of $\rank_{\lambda}(\matF(\hat{\vtheta}))$. When measured on the training set, $\trace(\matC(\hat{\vtheta)})/\trace(\matF(\hat{\vtheta)})$ indicates overfitting as it approaches zero, and we expect a decline in the ability of $\rank_{\lambda}(\matF(\hat{\vtheta}))$ to predict the gap. On the other hand, when measured on the test set, $|\trace(\matC(\hat{\vtheta)})/\trace(\matF(\hat{\vtheta)})-1|$ indicates calibration as it approaches zero, and we expect an improvement in the ability of $\rank_{\lambda}(\matF(\hat{\vtheta}))$ to predict the gap. The test was conducted by measuring the Kendall correlation in a moving window that contains 50\% of the data points, and the results appear to be as expected and are presented in Figure~\ref{fig:ratio_influence}.

\begin{figure}[htbp]
  \centering

    \centering
    \includegraphics[width=0.48\linewidth]{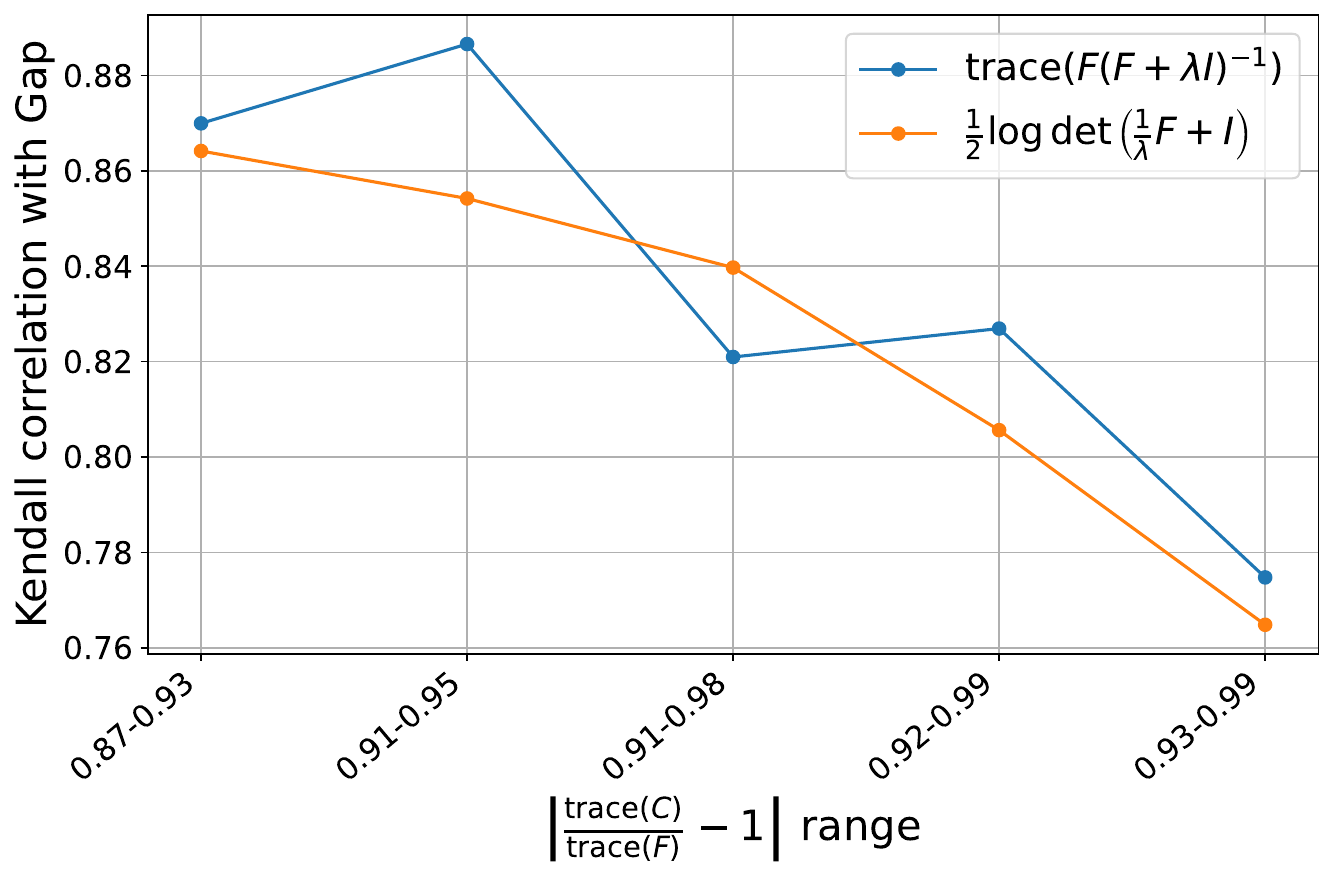}
    \hfill
    \includegraphics[width=0.48\linewidth]{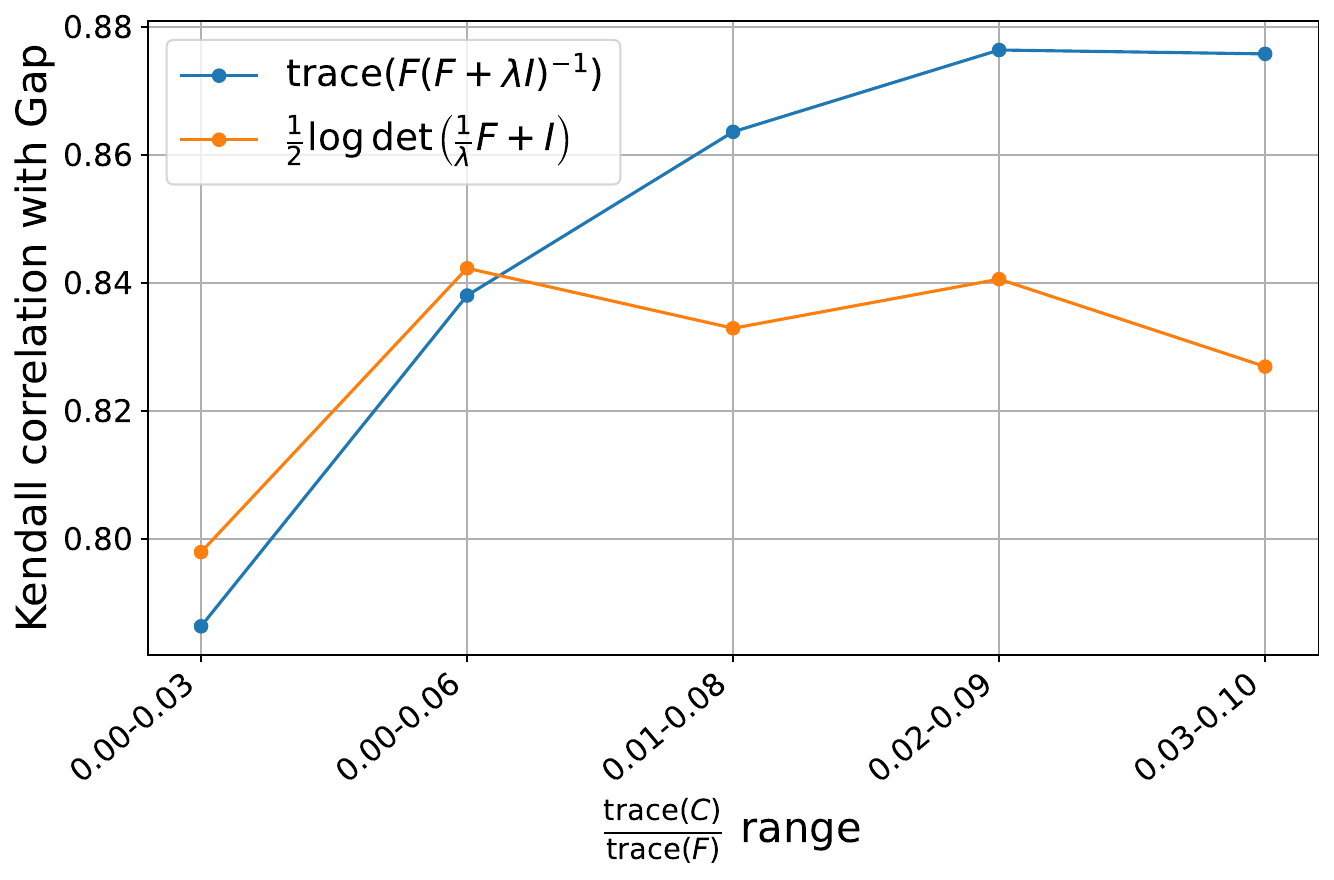}

    \caption{(Left) The ratio is measured on the test set. We see the relationship between model calibration
      (x‐axis: better calibration toward the left) and performance metrics,
      showing improved performance with better calibration.
      (Right) The ratio is measured on the train set.We see the  impact of overfitting (x‐axis: increased overfitting toward
      the left) on performance, highlighting degradation,
      especially for the soft-rank metric.}
    \label{fig:calibration_and_overfitting}
    \label{fig:ratio_influence}
\end{figure}



\subsection{Experimental Results Regarding The Relation Between C and F}
\label{app:experimental_results_CF}

In Section~\ref{sec:on_The_relation} we mentioned two empirical observations related to the relationship between the covariance and Fisher information matrices:

\begin{enumerate}
    \item The Jacobian $\matJ(\hat{\vtheta})$ exhibits independence from both $\matB(\hat{\theta})$ and $\matSigma_{\hat{\vtheta}}$, manifesting in the following approximations: 
\begin{align*}
   \matC(\hat{\vtheta}) &= \E_\vz\!\left[\matJ_{\hat{\vtheta}}^\T \matB_{\hat{\vtheta}}\matJ_{\hat{\vtheta}}\right] 
   \;\approx\; \E_\vz\!\left[\matJ_{\hat{\vtheta}}^\T \,\E_\vz\!\left[\matB_{\hat{\vtheta}}\right]\matJ_{\hat{\vtheta}}\right], \\
   \matF(\hat{\vtheta}) &= \E_\vz\!\left[\matJ_{\hat{\vtheta}}^\T \matSigma_{\hat{\vtheta}}\matJ_{\hat{\vtheta}}\right] 
   \;\approx\; \E_\vz\!\left[\matJ_{\hat{\vtheta}}^\T \,\E_\vz\!\left[\matSigma_{\hat{\vtheta}}\right]\matJ_{\hat{\vtheta}}\right],
\end{align*}
    
    \item The matrices $\E\left[\matB_{\hat{\vtheta}}\right]$ and $\E\left[\matSigma_{\hat{\vtheta}}\right]$ tend to be proportional.
\end{enumerate}
We now present the experimental results that led to these observations.

\begin{figure}[t]
    \centering
    \begin{minipage}{0.4\columnwidth}
        \includegraphics[width=\columnwidth]{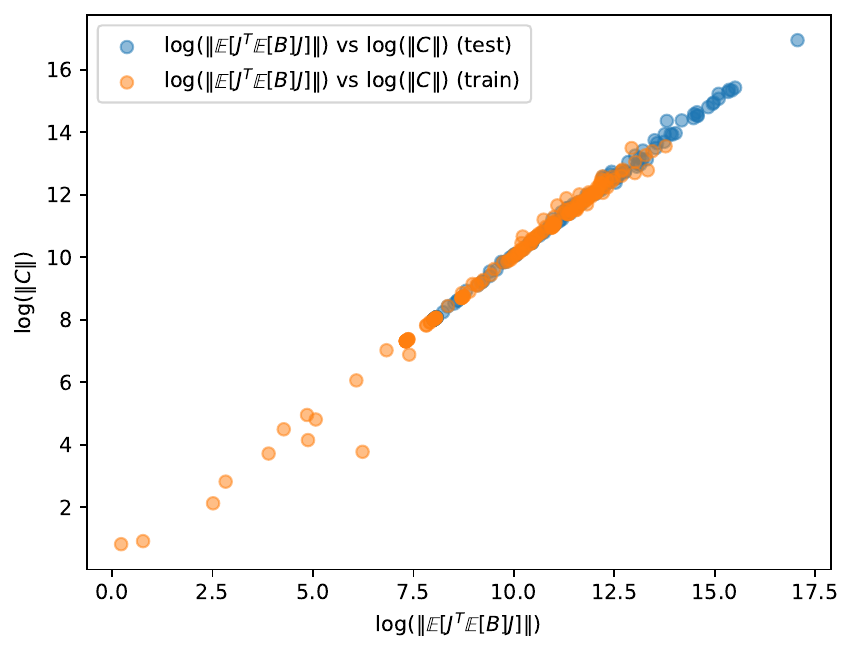}
    \end{minipage}
    \hfill
    \begin{minipage}{0.4\columnwidth}
        \includegraphics[width=\columnwidth]{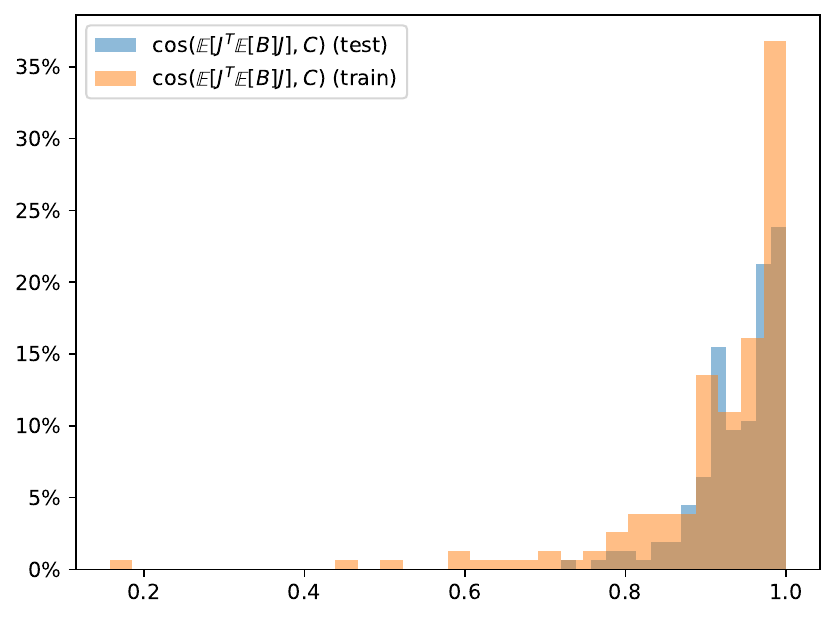}
    \end{minipage}
    \caption{Left: Comparison of matrix log norms showing strong correlation between $\matC(\hat{\vtheta})$ and $\E[\matJ_{\hat{\vtheta}}^\T\E[\matB_{\hat{\vtheta}}]\matJ_{\hat{\vtheta}}]$. Right: Histogram of cosine similarities between the matrices is concentrated near 1, indicating alignment}
    \label{fig:C_comparison}
\end{figure}

\begin{figure}[t]
    \centering
    \begin{minipage}{0.4\columnwidth}
        \includegraphics[width=\columnwidth]{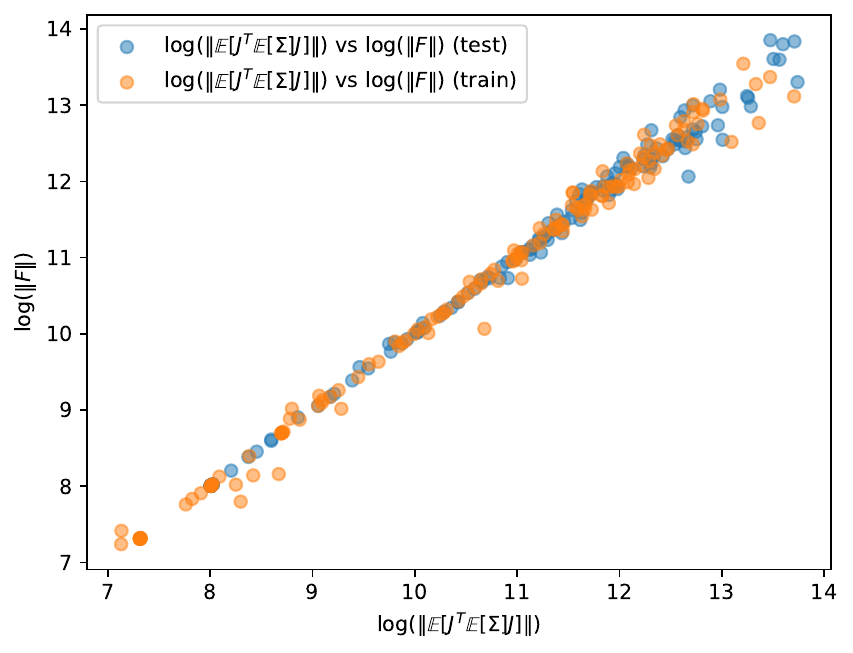}
    \end{minipage}
    \hfill
    \begin{minipage}{0.4\columnwidth}
        \includegraphics[width=\columnwidth]{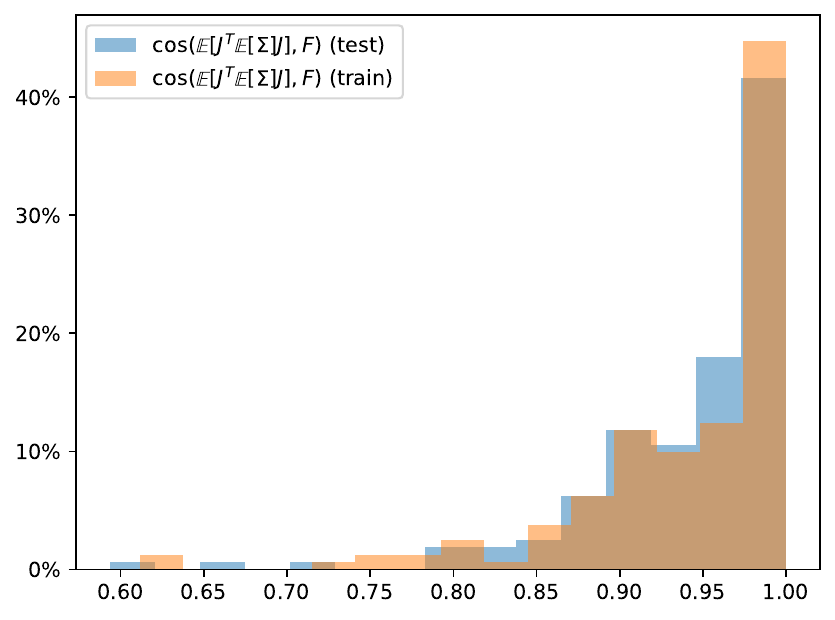}
    \end{minipage}
    \caption{Left: Log norm comparison between $\matF(\hat{\vtheta})$ and $\E[\matJ_{\hat{\vtheta}}^\T\E[\matSigma_{\hat{\vtheta}}]\matJ_{\hat{\vtheta}}]$, showing strong correlation. Right: Histogram of cosine similarities is concentrated near 1, indicating alignment.}
    \label{fig:F_comparison}
\end{figure}

\begin{figure}[t]
    \centering
    \includegraphics[width=0.5\columnwidth]{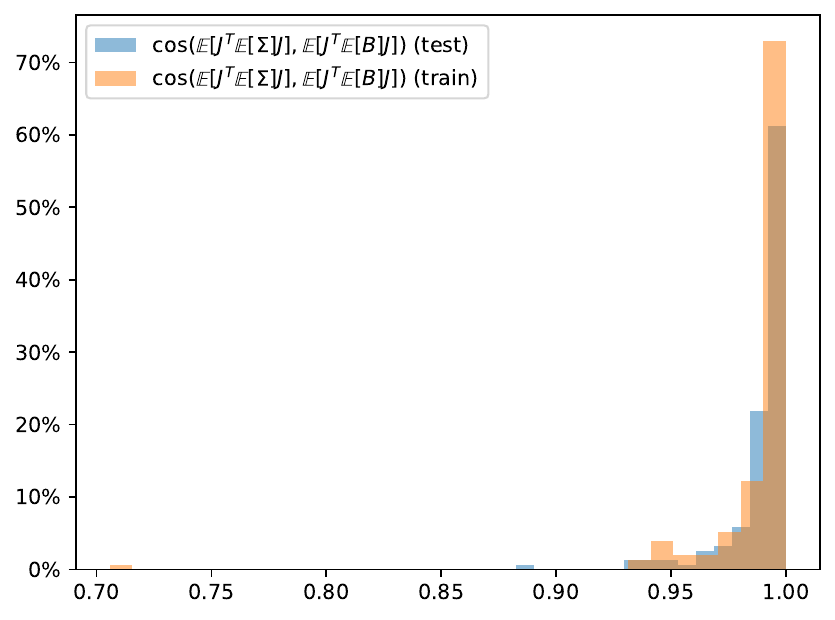}
    \caption{Histogram of cosine similarities between $\E[\matB_{\hat{\vtheta}}]$ and $\E[\matSigma_{\hat{\vtheta}}]$. The strong concentration of values near 1 indicates proportionality.}
    \label{fig:B_Sigma_comparison}
\end{figure}

Figure~\ref{fig:C_comparison} provides evidence for the first part of our independence claim regarding $\matC(\hat{\vtheta})$. The left panel demonstrates that the norms of $\matC(\hat{\vtheta})$ and its approximation $\E[\matJ_{\hat{\vtheta}}^\T\E[\matB_{\hat{\vtheta}}]\matJ_{\hat{\vtheta}}]$ are highly correlated. The right panel further shows that the cosine similarities between these matrices are usually close to 1, indicating that they tend to be proportional. Figure~\ref{fig:F_comparison} demonstrates the same phenomena regarding $\matF(\hat{\vtheta})$ and $\E\left[\matJ_{\hat{\vtheta}}^\T\E\left[\matSigma_{\hat{\vtheta}}\right]\matJ_{\hat{\vtheta}}\right]$.  

Figure~\ref{fig:B_Sigma_comparison} addresses our second claim regarding the proportionality between $\E[\matB_{\hat{\vtheta}}]$ and $\E[\matSigma_{\hat{\vtheta}}]$. The histogram of cosine similarities between these matrices shows a strong concentration near 1.

\section{Conclusion}
\label{sec:conclusion}

In this work, we revisited the connection between the flatness of local minima and generalization. We began by noting that \emph{sharpness}, rather than flatness, may be preferable when the goal is an accurate estimation of the optimal parameter $\vtheta^\star$. This perspective aligns with robust learning, particularly when the training distribution in feature space differs substantially from the test distribution.  

Using classical asymptotic analysis, we showed that, under calibration and practical assumptions, the appropriate measure of flatness in exponential family neural networks is the \emph{soft rank} of the Hessian. Our experiments confirmed that the Hessian soft rank is a far stronger predictor of generalization than the nuclear or operator norms, which are more commonly used in prior work.

Beyond the theoretically interesting case where the soft rank matches the asymptotic gap exactly, we also analyzed the non-calibrated case of softmax classification. Here, we showed that the generalization gap is governed by  
\[
n^{-1}\cdot\frac{\mathrm{MSE}(\vtheta^\star)}{\mathrm{Uncertainty}(\vtheta^\star)}\cdot\rank_\matLambda\!\left(\matF(\vtheta^\star)\right).
\]
Theoretically, we proved that the ratio between error and uncertainty cannot be estimated from the training set (using $\hat{\vtheta}$), as it tends to zero in the overfitting regime. In contrast, the soft rank factor avoids this limitation, and we demonstrated empirically that it can be reliably estimated from training data (via $\hat{\vtheta}$). Assuming  
\[
0 < a \leq \frac{\mathrm{MSE}(\vtheta^\star)}{\mathrm{Uncertainty}(\vtheta^\star)} \leq b < \infty,
\]  
we obtained a practical means of assessing the generalization gap through the soft rank of the Fisher Information Matrix. The accuracy of this assessment depends on the constants $a$ and $b$. Importantly, in the case of where mistakes are made with high confidence at test time, which might be related to overfitting (of $\vtheta^\star$), $b$ diverges, making it impossible to assess the generalization gap via the Hessian.  

From a computational standpoint, we exploited the concavity of the soft rank as a trace function. We proved that the soft rank of a block diagonal submatrix $\tilde{S}$ of a positive semi-definite matrix $S$ bounds the soft rank of $S$. We also showed that this concavity provides an upper bound when approximating the soft rank in practice, such as when using SGD batches. 

We further connected the soft rank to the Bayesian log-determinant complexity measure, noting that both are concave trace functions. This shared property enables similar approximation methods and leads to comparable results in practice.

Finally, our experiments demonstrated that the soft rank (and the related log-determinant) of the Fisher information matrix has a much stronger correlation with generalization than either the trace or the maximal eigenvalue of the Fisher information. These correlations remain robust even under inexpensive approximations such as KFAC and diagonal methods.

\bibliography{flatness}

\begin{thebibliography}{43}
\providecommand{\natexlab}[1]{#1}
\providecommand{\url}[1]{\texttt{#1}}
\expandafter\ifx\csname urlstyle\endcsname\relax
  \providecommand{\doi}[1]{doi: #1}\else
  \providecommand{\doi}{doi: \begingroup \urlstyle{rm}\Url}\fi

\bibitem[Akaike(1974)]{akaike1974new}
Hirotugu Akaike.
\newblock A new look at the statistical model identification problem.
\newblock \emph{IEEE Trans. Autom. Control}, 19:\penalty0 716, 1974.

\bibitem[Andriushchenko et~al.(2023)Andriushchenko, Croce, M{\"u}ller, Hein, and Flammarion]{andriushchenko2023modern}
Maksym Andriushchenko, Francesco Croce, Maximilian M{\"u}ller, Matthias Hein, and Nicolas Flammarion.
\newblock A modern look at the relationship between sharpness and generalization.
\newblock In \emph{International Conference on Machine Learning}, pages 840--902. PMLR, 2023.

\bibitem[Arora et~al.(2019)Arora, Cohen, Hu, and Luo]{arora2019implicit}
Sanjeev Arora, Nadav Cohen, Wei Hu, and Yuping Luo.
\newblock Implicit regularization in deep matrix factorization.
\newblock \emph{Advances in Neural Information Processing Systems}, 32, 2019.

\bibitem[Baratin et~al.(2021)Baratin, George, Laurent, Hjelm, Lajoie, Vincent, and Lacoste-Julien]{baratin2021implicit}
Aristide Baratin, Thomas George, C{\'e}sar Laurent, R~Devon Hjelm, Guillaume Lajoie, Pascal Vincent, and Simon Lacoste-Julien.
\newblock Implicit regularization via neural feature alignment.
\newblock In \emph{International Conference on Artificial Intelligence and Statistics}, pages 2269--2277. PMLR, 2021.

\bibitem[Bernstein(2011)]{bernstein11}
Dennis~S. Bernstein.
\newblock \emph{Matrix Mathematics: Theory, Facts, and Formulas}.
\newblock Princeton University Press, second edition, 2011.
\newblock ISBN 9781400833344 1400833345 1283256010 9781283256018.

\bibitem[Cogswell et~al.(2016)Cogswell, Ahmed, Girshick, Zitnick, and Batra]{cogswell2015reducing}
Michael Cogswell, Faruk Ahmed, Ross~B. Girshick, Larry Zitnick, and Dhruv Batra.
\newblock Reducing overfitting in deep networks by decorrelating representations.
\newblock In Yoshua Bengio and Yann LeCun, editors, \emph{4th International Conference on Learning Representations, {ICLR} 2016, San Juan, Puerto Rico, May 2-4, 2016, Conference Track Proceedings}, 2016.
\newblock URL \url{http://arxiv.org/abs/1511.06068}.

\bibitem[Ding et~al.(2024)Ding, Drusvyatskiy, Fazel, and Harchaoui]{ding2024flat}
Lijun Ding, Dmitriy Drusvyatskiy, Maryam Fazel, and Zaid Harchaoui.
\newblock Flat minima generalize for low-rank matrix recovery.
\newblock \emph{Information and Inference: A Journal of the IMA}, 13\penalty0 (2):\penalty0 iaae009, 2024.

\bibitem[Dinh et~al.(2017)Dinh, Pascanu, Bengio, and Bengio]{dinh2017sharp}
Laurent Dinh, Razvan Pascanu, Samy Bengio, and Yoshua Bengio.
\newblock Sharp minima can generalize for deep nets.
\newblock In \emph{International Conference on Machine Learning}, pages 1019--1028. PMLR, 2017.

\bibitem[Gastpar et~al.(2024)Gastpar, Nachum, Shafer, and Weinberger]{gastpar2024fantastic}
Michael Gastpar, Ido Nachum, Jonathan Shafer, and Thomas Weinberger.
\newblock Fantastic generalization measures are nowhere to be found.
\newblock In \emph{The Twelfth International Conference on Learning Representations}, 2024.
\newblock URL \url{https://openreview.net/forum?id=NkmJotfL42}.

\bibitem[Gatmiry et~al.(2024)Gatmiry, Li, Ma, Reddi, Jegelka, and Chuang]{gatmiry2024inductive}
Khashayar Gatmiry, Zhiyuan Li, Tengyu Ma, Sashank Reddi, Stefanie Jegelka, and Ching-Yao Chuang.
\newblock What is the inductive bias of flatness regularization? {A} study of deep matrix factorization models.
\newblock \emph{Advances in Neural Information Processing Systems}, 36, 2024.

\bibitem[Hua et~al.(2021)Hua, Wang, Xue, Ren, Wang, and Zhao]{hua2021feature}
Tianyu Hua, Wenxiao Wang, Zihui Xue, Sucheng Ren, Yue Wang, and Hang Zhao.
\newblock On feature decorrelation in self-supervised learning.
\newblock In \emph{Proceedings of the IEEE/CVF International Conference on Computer Vision}, pages 9598--9608, 2021.

\bibitem[Huh et~al.(2023)Huh, Mobahi, Zhang, Cheung, Agrawal, and Isola]{huh2023the}
Minyoung Huh, Hossein Mobahi, Richard Zhang, Brian Cheung, Pulkit Agrawal, and Phillip Isola.
\newblock The low-rank simplicity bias in deep networks.
\newblock \emph{Transactions on Machine Learning Research}, 2023.
\newblock ISSN 2835-8856.
\newblock URL \url{https://openreview.net/forum?id=bCiNWDmlY2}.

\bibitem[Immer et~al.(2021{\natexlab{a}})Immer, Bauer, Fortuin, R{\"a}tsch, and Emtiyaz]{immer2021scalable}
Alexander Immer, Matthias Bauer, Vincent Fortuin, Gunnar R{\"a}tsch, and Khan~Mohammad Emtiyaz.
\newblock Scalable marginal likelihood estimation for model selection in deep learning.
\newblock In \emph{International Conference on Machine Learning}, pages 4563--4573. PMLR, 2021{\natexlab{a}}.

\bibitem[Immer et~al.(2021{\natexlab{b}})Immer, Korzepa, and Bauer]{immer2021improving}
Alexander Immer, Maciej Korzepa, and Matthias Bauer.
\newblock Improving predictions of {B}ayesian neural nets via local linearization.
\newblock In \emph{International conference on artificial intelligence and statistics}, pages 703--711. PMLR, 2021{\natexlab{b}}.

\bibitem[Immer et~al.(2022)Immer, van~der Ouderaa, R{\"a}tsch, Fortuin, and van~der Wilk]{immer2022invariance}
Alexander Immer, Tycho van~der Ouderaa, Gunnar R{\"a}tsch, Vincent Fortuin, and Mark van~der Wilk.
\newblock Invariance learning in deep neural networks with differentiable {L}aplace approximations.
\newblock \emph{Advances in Neural Information Processing Systems}, 35:\penalty0 12449--12463, 2022.

\bibitem[Immer et~al.(2023)Immer, Van Der~Ouderaa, Van Der~Wilk, Ratsch, and Sch{\"o}lkopf]{immer2023stochastic}
Alexander Immer, Tycho~FA Van Der~Ouderaa, Mark Van Der~Wilk, Gunnar Ratsch, and Bernhard Sch{\"o}lkopf.
\newblock Stochastic marginal likelihood gradients using neural tangent kernels.
\newblock In \emph{International Conference on Machine Learning}, pages 14333--14352. PMLR, 2023.

\bibitem[Jang et~al.(2022)Jang, Lee, Park, and Noh]{jang2022reparametrization}
Cheongjae Jang, Sungyoon Lee, Frank Park, and Yung-Kyun Noh.
\newblock A reparametrization-invariant sharpness measure based on information geometry.
\newblock \emph{Advances in neural information processing systems}, 35:\penalty0 27893--27905, 2022.

\bibitem[Karamata(1932)]{karamata1932inegalite}
Jovan Karamata.
\newblock Sur une in{\'e}galit{\'e} relative aux fonctions convexes.
\newblock \emph{Publications de l'Institut mathematique}, 1\penalty0 (1):\penalty0 145--147, 1932.

\bibitem[Kim et~al.(2022)Kim, Li, Hu, and Hospedales]{kim22fisher}
Minyoung Kim, Da~Li, Shell~X Hu, and Timothy Hospedales.
\newblock {F}isher {SAM}: Information geometry and sharpness aware minimisation.
\newblock In Kamalika Chaudhuri, Stefanie Jegelka, Le~Song, Csaba Szepesvari, Gang Niu, and Sivan Sabato, editors, \emph{Proceedings of the 39th International Conference on Machine Learning}, volume 162 of \emph{Proceedings of Machine Learning Research}, pages 11148--11161. PMLR, 17--23 Jul 2022.
\newblock URL \url{https://proceedings.mlr.press/v162/kim22f.html}.

\bibitem[Kristiadi et~al.(2024)Kristiadi, Dangel, and Hennig]{kristiadi2024geometry}
Agustinus Kristiadi, Felix Dangel, and Philipp Hennig.
\newblock The geometry of neural nets' parameter spaces under reparametrization.
\newblock \emph{Advances in Neural Information Processing Systems}, 36, 2024.

\bibitem[Krizhevsky and Hinton(2009)]{krizhevsky2009learning}
Alex Krizhevsky and Geoffrey Hinton.
\newblock Learning multiple layers of features from tiny images.
\newblock In \emph{Technical Report}, volume~1, page~7. University of Toronto, 2009.

\bibitem[Kunstner et~al.(2019)Kunstner, Hennig, and Balles]{kunstner2019limitations}
Frederik Kunstner, Philipp Hennig, and Lukas Balles.
\newblock Limitations of the empirical fisher approximation for natural gradient descent.
\newblock \emph{Advances in neural information processing systems}, 32, 2019.

\bibitem[LeCun(1998)]{lecun1998mnist}
Yann LeCun.
\newblock The mnist database of handwritten digits.
\newblock \emph{http://yann. lecun. com/exdb/mnist/}, 1998.

\bibitem[Lee et~al.(2023)Lee, Lee, Hwang, Lee, Lee, and Choo]{lee2023importance}
Hojoon Lee, Koanho Lee, Dongyoon Hwang, Hyunho Lee, Byungkun Lee, and Jaegul Choo.
\newblock On the importance of feature decorrelation for unsupervised representation learning in reinforcement learning.
\newblock In \emph{International Conference on Machine Learning}, pages 18988--19009. PMLR, 2023.

\bibitem[Liu et~al.(2023)Liu, Xie, Li, and Ma]{liu2023same}
Hong Liu, Sang~Michael Xie, Zhiyuan Li, and Tengyu Ma.
\newblock Same pre-training loss, better downstream: Implicit bias matters for language models.
\newblock In \emph{International Conference on Machine Learning}, pages 22188--22214. PMLR, 2023.

\bibitem[Liu(1995)]{liu1995unbiased}
Yong Liu.
\newblock Unbiased estimate of generalization error and model selection in neural network.
\newblock \emph{Neural Networks}, 8\penalty0 (2):\penalty0 215--219, 1995.

\bibitem[Lundquist and Barrett(1996)]{lundquist-1996-rank-inequal}
Michael Lundquist and Wayne Barrett.
\newblock Rank inequalities for positive semidefinite matrices.
\newblock \emph{Linear Algebra and its Applications}, 248:\penalty0 91--100, 1996.
\newblock ISSN 0024-3795.
\newblock \doi{10.1016/0024-3795(95)00170-0}.
\newblock URL \url{https://www.sciencedirect.com/science/article/pii/0024379595001700}.

\bibitem[MacKay(1992{\natexlab{a}})]{mackay1992thesis}
D.~J. MacKay.
\newblock \emph{Bayesian methods for adaptive models}.
\newblock PhD thesis, California Institute of Technology, 1992{\natexlab{a}}.

\bibitem[MacKay(1992{\natexlab{b}})]{mackay1992bayesian}
David J.~C. MacKay.
\newblock Bayesian model comparison and backprop nets.
\newblock In \emph{Advances in Neural Information Processing Systems 4 (NIPS 1991)}, number~4 in Advances in Neural Information Processing Systems, pages 839--846, San Mateo, CA, 1992{\natexlab{b}}. Morgan Kaufmann.
\newblock ISBN 1-55860-222-4.

\bibitem[Marshall et~al.(2010)Marshall, Olkin, and Arnold]{Marshall2010Matrix}
Albert~W. Marshall, Ingram Olkin, and Barry~C. Arnold.
\newblock In \emph{Inequalities: Theory of Majorization and Its Applications}. Springer New York, 2010.

\bibitem[Martens(2020)]{martens2020new}
James Martens.
\newblock New insights and perspectives on the natural gradient method.
\newblock \emph{The Journal of Machine Learning Research}, 21\penalty0 (1):\penalty0 5776--5851, 2020.

\bibitem[Martens and Grosse(2015)]{martens2015optimizing}
James Martens and Roger Grosse.
\newblock Optimizing neural networks with {Kronecker-Factored Approximate Curvature}.
\newblock In \emph{International conference on machine learning}, pages 2408--2417. PMLR, 2015.

\bibitem[Moody(1991)]{moody1991effective}
John Moody.
\newblock The effective number of parameters: An analysis of generalization and regularization in nonlinear learning systems.
\newblock \emph{Advances in neural information processing systems}, 4, 1991.

\bibitem[Murata et~al.(1994)Murata, Yoshizawa, and Amari]{murata1994network}
Noboru Murata, Shuji Yoshizawa, and Shun-ichi Amari.
\newblock Network information criterion-determining the number of hidden units for an artificial neural network model.
\newblock \emph{IEEE transactions on neural networks}, 5\penalty0 (6):\penalty0 865--872, 1994.

\bibitem[Naganuma et~al.(2022)Naganuma, Suzuki, Yokota, Nomura, Ishikawa, and Sato]{naganuma2022takeuchis}
Hiroki Naganuma, Taiji Suzuki, Rio Yokota, Masahiro Nomura, Kohta Ishikawa, and Ikuro Sato.
\newblock Takeuchi's information criteria as generalization measures for {DNN}s close to {NTK} regime, 2022.
\newblock URL \url{https://openreview.net/forum?id=FH_mZOKFX-b}.

\bibitem[Netzer et~al.(2011)Netzer, Wang, Coates, Bissacco, Wu, and Ng]{netzer2011reading}
Yuval Netzer, Tao Wang, Adam Coates, Alessandro Bissacco, Bo~Wu, and Andrew~Y Ng.
\newblock Reading digits in natural images with unsupervised feature learning.
\newblock In \emph{NIPS Workshop on Deep Learning and Unsupervised Feature Learning}, volume 2011, page~5, 2011.

\bibitem[Newey and McFadden(1994)]{newey1994large}
Whitney~K. Newey and Daniel McFadden.
\newblock {Chapter 36: Large} sample estimation and hypothesis testing.
\newblock volume~4 of \emph{Handbook of Econometrics}, pages 2111--2245. Elsevier, 1994.
\newblock \doi{https://doi.org/10.1016/S1573-4412(05)80005-4}.
\newblock URL \url{https://www.sciencedirect.com/science/article/pii/S1573441205800054}.

\bibitem[Petzka et~al.(2021)Petzka, Kamp, Adilova, Sminchisescu, and Boley]{petzka2021relative}
Henning Petzka, Michael Kamp, Linara Adilova, Cristian Sminchisescu, and Mario Boley.
\newblock Relative flatness and generalization.
\newblock \emph{Advances in neural information processing systems}, 34:\penalty0 18420--18432, 2021.

\bibitem[Razin and Cohen(2020)]{razin2020implicit}
Noam Razin and Nadav Cohen.
\newblock Implicit regularization in deep learning may not be explainable by norms.
\newblock \emph{Advances in Neural Information Processing Systems}, 33:\penalty0 21174--21187, 2020.

\bibitem[Shibata(1989)]{Shibata1989}
Ritei Shibata.
\newblock \emph{Statistical Aspects of Model Selection}, pages 215--240.
\newblock Springer Berlin Heidelberg, Berlin, Heidelberg, 1989.
\newblock ISBN 978-3-642-75007-6.
\newblock \doi{10.1007/978-3-642-75007-6_5}.
\newblock URL \url{https://doi.org/10.1007/978-3-642-75007-6_5}.

\bibitem[Takeuchi(1976)]{takeuchi1976distribution}
K~Takeuchi.
\newblock Distribution of information number statistics and criteria for adequacy of models.
\newblock \emph{Mathematical Sciences}, 153:\penalty0 12--18, 1976.

\bibitem[Thomas et~al.(2020)Thomas, Pedregosa, Merri{\"e}nboer, Manzagol, Bengio, and Le~Roux]{thomas2020interplay}
Valentin Thomas, Fabian Pedregosa, Bart Merri{\"e}nboer, Pierre-Antoine Manzagol, Yoshua Bengio, and Nicolas Le~Roux.
\newblock On the interplay between noise and curvature and its effect on optimization and generalization.
\newblock In \emph{International Conference on Artificial Intelligence and Statistics}, pages 3503--3513. PMLR, 2020.

\bibitem[Wang et~al.(2017)Wang, Fang, and Liu]{wang2017stochastic}
Mengdi Wang, Ethan~X Fang, and Han Liu.
\newblock Stochastic compositional gradient descent: algorithms for minimizing compositions of expected-value functions.
\newblock \emph{Mathematical Programming}, 161:\penalty0 419--449, 2017.

\end{thebibliography}
\bibliographystyle{plainnat}

\newpage
\appendix
\section*{\centering Appendix: Flatness After All?}

\section{Additional Matrix Algebra Lemmas}
Here we collect a few additional matrix algebra lemmas necessary for our proofs.
\begin{lemma}
    \label{lemma:ABA}
    Suppose that $\matA$ is positive definite, and $\matB$ is positive semidefinite. Further assume, $\matB^{\nicefrac{1}{2}}\matA^{-1}\matB^{\nicefrac{1}{2}} \preceq \matI$. Then $\matA^{-1}\matB\matA^{-1}\preceq \matA^{-1}$.
\end{lemma}
\begin{proof}
    Due to \cite[Fact 8.10.19]{bernstein11}, $\matA^{-1}\matB\matA^{-1}\preceq \matA^{-1}$ if and only if $\matA^{-1}\matB\matA^{-1}\matB\matA^{-1}\preceq \matA^{-1}\matB\matA^{-1}$. This, in turn, holds if and only if $\matB \matA^{-1} \matB \preceq \matB$. The prerequisite that $\matB^{\nicefrac{1}{2}}\matA^{-1}\matB^{\nicefrac{1}{2}} \preceq \matI$ implies that $\matB \matA^{-1} \matB \preceq \matB$, concluding the proof.
\end{proof}
\begin{lemma}
\label{lemma:ABC}
Let $\matA$, $\matB$, and $\matC$ be positive semidefinite matrices with
\[
\matA - \matB \succeq 0.
\]
Then
\[
\trace\left(((\matA - \matB)\matC)^2\right) \le \trace\left((\matA\matC)^2\right).
\]
\end{lemma}

\begin{proof}
Since \( \matA - \matB \preceq \matA \) and \( \matC \succeq 0 \), a congruence transformation by \( \matC^{1/2} \) yields
\[
\matC^{\nicefrac{1}{2}}(\matA - \matB)\matC^{\nicefrac{1}{2}} \preceq \matC^{\nicefrac{1}{2}} \matA \matC^{\nicefrac{1}{2}}.
\]
For two positive definite matrices $\matX$ and $\matY$, $\matX \preceq \matY$ implies that $\trace(\matX^2)\leq\trace(\matY^2)$ (this can be seen from the fact that $\trace(\matY^2) - \trace(\matX^2) = \trace((\matY - \matX)(\matY + \matX))$). Applying to both sides above, we have
$$
\trace\left(\matC^{\nicefrac{1}{2}}(\matA - \matB)\matC(\matA - \matB)\matC^{\nicefrac{1}{2}}\right) \leq \trace(\matC^{\nicefrac{1}{2}} \matA \matC \matA \matC^{\nicefrac{1}{2}}).$$
The claim now follows from the cyclicity of the trace operator.
\end{proof}

\end{document}